\documentclass[10pt]{article}

\usepackage[T1]{fontenc}
\usepackage{amsmath, amsthm, amssymb, amsfonts}
\usepackage{mathtools}
\usepackage{enumitem}
\usepackage[svgnames]{xcolor}
\usepackage{bbm}
\usepackage[margin=1in]{geometry}
\usepackage{mathrsfs}
\usepackage{booktabs}
\usepackage{tikz-cd}
\usepackage{graphicx}
\usepackage{extpfeil}
\usepackage{natbib}
\usepackage{subcaption}

\mathtoolsset{showonlyrefs}

\usepackage[
    textwidth=3cm,
    textsize=small,
    colorinlistoftodos
]{todonotes}

\usepackage{hyperref}
\usepackage[capitalize]{cleveref}

\newcommand{\R}{\mathbb{R}}

\DeclareMathOperator*{\E}{\mathbb{E}}

\renewcommand{\epsilon}{\varepsilon}
\renewcommand{\phi}{\varphi}

\renewcommand{\dh}{d_\mathrm{h}}
\newcommand{\din}{d_\mathrm{in}}
\newcommand{\xemb}{\mathbf{x}_{\mathrm{emb}}}

\numberwithin{equation}{section} 

\theoremstyle{plain}
\newtheorem{theorem}{Theorem}

\newtheorem{lemma}{Lemma}

\newtheorem{assumption}{Assumption}
\newtheorem{condition}{Condition}

\theoremstyle{definition}

\theoremstyle{remark}
\newtheorem{remark}{Remark}

\crefname{lemma}{Lemma}{Lemmas}
\crefname{theorem}{Theorem}{Theorems}
\crefname{prop}{Proposition}{Propositions}
\crefname{corollary}{Corollary}{Corollaries}
\crefname{defn}{Definition}{Definitions}
\crefname{assumption}{Assumption}{Assumptions}
\crefname{remark}{Remark}{Remarks}
\crefname{equation}{Equation}{Equations}

\begin{document}

\title{Training-Free Universal Approximation by Prompting \\ Random Transformers}
\date{}
\author{Alexander Hsu\thanks{A. Hsu (hsu297@purdue.edu) is with the Department of Mathematics, Purdue University.} \and Rongjie Lai\thanks{R. Lai (lairj@purdue.edu) is with the Department of Mathematics, Purdue University.}}
\maketitle

\begin{abstract}
How expressive is prompting a transformer? Answering this question is important for separating the roles of prompting, architecture, and pretraining in transformer models, and for determining whether task-specific behavior must be stored in model weights or can instead be induced at inference time through the prompt. We show, in an approximation-theoretic sense, that pretraining is optional: a single-layer softmax attention network with random, untrained weights can approximate any Hölder function on a compact manifold when steered by an appropriate soft prompt. Guided by the connection between softmax attention and kernel methods, we construct explicit soft prompts—a prompt per target function, independent of the query—as solutions to linear systems matching attention logits to Gaussian kernel exponents, under which the frozen transformer emulates the classical Nadaraya-Watson kernel estimator. The construction requires only a mild rank condition on the weights, which we show holds almost surely under Gaussian initialization. The prompted network inherits the theoretical guarantees of kernel regression, leading to universal approximation theorems with minimax-optimal rates that depend on the intrinsic dimension. We further quantify the cost of prompting, exposing a tradeoff between the norm of the constructed soft prompt tokens, prompt length, and hidden dimension. Numerical experiments corroborate the constructions and predicted rates. 
\end{abstract}

\section{Introduction}\label{sec:intro}

The transformer architecture \citep{Vaswani17} has seen tremendous real-world success, most notably through the rise of the large language model (LLM). These models are massive in scale, with modern frontier models using trillions of parameters and petabytes of training data \citep{Zhao26}. Given the prohibitively expensive pretraining costs, significant efforts have been dedicated to getting more out of existing pretrained models, rather than retraining from scratch. Ideally, general-purpose transformers could generalize to diverse domains, incorporate new information, or even self-improve. 

A high-level approach, commonly employed in agent-forward settings, extends capabilities by allowing transformers access to external resources at inference time via methods such as \emph{retrieval-augmented generation} \citep{Lewis20} or \emph{tool calling} \citep{Schick23}. On the architectural level, (weight) \emph{fine-tuning} and its many variants have proved popular, which work by adapting a subset of weights to different downstream tasks \citep{Houlsby19,Hu22}. 

The focus of this work is to explore \emph{prompt engineering}, broadly construed, which changes the behavior of the model through its inputs \citep{Sahoo24}. Altering the prompt has emerged as a lightweight, flexible method of adapting pretrained transformers to new tasks at inference time. The most straightforward method directly edits the human-language input fed into the model, known as \emph{hard prompting}. Some examples include providing examples in the prompt (\emph{few-shot prompting}, \citet{Brown20}), requiring the model to think step-by-step (\emph{Chain-of-Thought}, \citet{Wei22}), or telling the model to assume a specific role (\emph{role-prompting}, \citet{Kong24}). While these work well heuristically, hard prompting is formally a discrete optimization problem, which makes it difficult from an optimization standpoint \citep{Shin20}. More amenable to mathematical techniques is \emph{soft prompting} (or \emph{prompt tuning}, \citet{Lester21}), which optimizes a prefix in real-valued embedding space, and \emph{prefix-tuning} \citep{Li21}, which appends trainable tokens to the keys and values of every attention layer. Compared to weight-based tuning, these methods tend to require fewer trainable parameters and are easier to edit on the fly \citep{Lester21,Li21}. Despite its widespread adoption, theoretical guarantees for prompt engineering are still developing. 

Questions traditionally asked of neural network weights can equally be asked of prompts. For example, it is natural to wonder how expressive prompting methods are, especially in comparison to weight-based methods \citep{Petrov24b,Meyer25}. Viewing transformers through the lens of control theory \citep{Luo23}, the expressivity of prompting corresponds to understanding reachability \citep{Bhargava23}. Our work focuses on soft prompting in particular, where we formulate the question of expressivity formally through the ideas of approximation theory. One particularly relevant question is that of \emph{universal approximation}: can a transformer with fixed weights be steered to approximate any function in some class simply by changing its prompt? In certain settings, this has been answered in the affirmative. \citet{Wang23} and \citet{Hu25} show prompting can universally approximate Lipschitz sequence-to-sequence functions, and \citet{Petrov24a} derive corresponding quantitative approximation rates for prompting a single attention head model. \citet{Nakada25} study a mechanism behind prompting-based universal approximation, showing prompts can program a transformer to emulate ReLU networks at inference. All these works demonstrate the potential and limitations of prompting, but suffer from a limitation common to approximation-theoretic works: the results rely on particular transformer weights, for which there are no guarantees the training process will find or even approximate. The training dynamics of transformers is an important (and difficult) question, which is only beginning to be studied \citep{Ahn23,Oymak23,Huang24,Zhang24}. 

We consider function regression, but eschew training completely by studying random transformers with frozen, untrained weights. Our work culminates in \cref{thm:univapproxrandomtransformer}, which concerns a very simple transformer model consisting of an affine embedding layer, one softmax single-headed attention layer, and an affine decoding layer, all of which have random Gaussian weights. Even in this extreme setting, we constructively show that for any given H\"{o}lder function in some class, there almost surely exists a soft prompt which, when prepended to any query, makes the transformer approximate the function evaluated at that query, with approximation rate scaling with prompt length. Moreover, we track the constructed soft prompt magnitude, exposing a tradeoff involving prompt length and hidden dimension. 

Our approximation approach is straightforward and intuitive, building on connections between the attention mechanism \citep{Bahdanau15,Vaswani17} and kernel methods \citep{Tsai19,Shen26,Ching26}. In particular, we show that transformers can reproduce the Nadaraya-Watson kernel estimator \citep{Nadaraya64,Watson64} in their forward pass, similar to \citet{Shen26}. Whereas they choose specific weights to algebraically build the estimator through the transformer blocks, we show that an appropriate soft prompt can already force a single softmax attention layer to closely match the estimator's functional form. The soft prompt is derived as the solution to specific linear systems, and we specify a mild rank condition on the weights (\cref{ass:parameters}) under which a solution exists. Using standard techniques from random matrix theory, we show that random Gaussian initializations satisfy this condition almost surely. 

Working through the well-studied Nadaraya-Watson estimator allows us to borrow existing approximation theory bounds for kernel regression. We extend the approximation of the kernel estimator through the soft prompt to full quantitative universal approximation-type theorems for prompting with rate of convergence (Jackson-type), with probabilistic guarantees in the nondeterministic cases. Our results also inherit the beneficial properties of the kernel estimator, such as minimax-optimal rates for noisy regression of H\"{o}lder functions on manifolds \citep{Bickel07} up to logarithmic factors, with primary dependence only on the manifold's intrinsic dimension.

We summarize our results as follows:
\begin{enumerate}
    \item We show that prompting can reproduce the Nadaraya-Watson kernel estimator: given a single-layer softmax attention transformer whose weights satisfy a mild rank condition, and any set of $n$ support points, we construct a soft prompt of prompt length $n$ under which the transformer's output approximates the NW kernel estimator built on those points (\cref{thm:promptingkernel}). The soft prompt is the minimum-norm solution of a linear system matching attention logits to Gaussian kernel exponents; we use \emph{logit shifting} to suppress the unwanted query self-attention term at an exponential rate. 

    \item We show that the rank condition for the weights is generic, satisfied by Gaussian initializations almost surely, which allows the construction to work with almost every random-attention transformer (\cref{thm:promptingkernelrandomattn}). Furthermore, we extend to fully random transformers, where embedding and decoding become affine maps with Gaussian weights (\cref{thm:promptingkernelrandomtransformer}). We conclude similar results in this setting.

    \item In each setting, we produce corresponding quantitative universal approximation theorems for H\"{o}lder functions on a manifold (\cref{thm:univapprox,thm:univapproxrandomattn,thm:univapproxrandomtransformer}). By tuning the kernel bandwidth and logit shift, we achieve the minimax-optimal noisy approximation rate of $\tilde{O}(n^{-2\alpha/(2\alpha+d_X)})$, where $\alpha$ is the H\"{o}lder exponent and $d_X$ is the intrinsic manifold dimension. 

    \item We track the costs of our constructions in terms of the magnitude of the soft prompt tokens, and how it relates to the prompt length $n$ and hidden dimension $\dh$. Prompt token norms grow as $O(n^{2/(2\alpha+d_X)})$ in the deterministic parameter case with fixed embedding and decoding, improving to $O(n^{2/(2\alpha+d_X)}/\sqrt{\dh}+1)$ with random attention, and $O(n^{2/(2\alpha+d_X)}/\dh+\sqrt{\dh})$ for fully random transformers. 
\end{enumerate}

\subsection{Other Related Works}

\paragraph{Single-task and in-context learning}
Traditional models are studied primarily in the setting of \emph{single-task learning}, where each task or function requires its own tailored model. The ability of neural networks to approximate a single function is the foundation of approximation theory in machine learning, and has been studied extensively for various architectures \citep{Cybenko89,Hornik91,Barron93,Yarotsky17}. For transformers in particular, universal approximation was first shown by \citet{Yun20}, and later extended to other settings \citep{Takakura23,Havrilla24,Jiang24}. Relevant to our setting is the expressivity of models with only one attention layer, which has been studied by \citet{Kajitsuka24} and \citet{Jiang24}.

In contrast, transformers also demonstrate an additional remarkable phenomenon known as \emph{in-context learning} (ICL), where a transformer is able to process and apply information provided in the prompt, entirely at inference time with frozen weights \citep{Radford19,Brown20}. The theoretical setting usually considers prompts consisting of a function, described as input-output pairs, which forms the \emph{context}, followed by a separate \emph{query} \citep{Garg22}. ICL can be thought of as the transformer's ability to approximate an operator which maps such prompts to the function evaluated on the query. It must infer the relevant function from the context, then apply it to the query, and output the result, all in the same forward pass. A popular mechanistic explanation of ICL is that transformers, with the appropriate weights, are able to execute certain statistical algorithms at inference, such as gradient descent \citep{Akyurek23,vonOswald23,Ahn23}, direct least-squares \citep{Akyurek23,Bai23}, polynomial and spline regression \citep{Hsu26}, and kernel methods \citep{Han25,Shen26}. Related to our work, \citet{Hu26} show that algorithm parameters can be encoded into the prompt, such that the transformer then emulates the corresponding algorithm in-context.

While our universal approximation theorems are presented in the language of single-task learning, we overlap heavily with the ICL setting. Most notably, all our results involve first fixing a transformer (i.e., a particular realization of the weights), which we then use to regress any function of a certain class by only adapting the soft prompt. The soft prompt is constructed based on the specific realization of the weights and function, but is independent of the chosen query. Much like context in the theoretical ICL setting, it is prepended to the query, which then steers the transformer to perform kernel regression. The main difference is that ICL assumes a fixed algebraic form of the input sequence (input-output pairs followed by query) and chooses the weights to process those, whereas we fix the weights and find a task-dependent soft prompt to achieve the same goal.

\paragraph{Randomness in machine learning}
Many techniques in machine learning leverage randomness. Beyond stochasticity in training, frozen random parameters have proven surprisingly effective and insightful across many domains, while also saving on the computational cost of training. \citet{Rahimi07,Rahimi08} demonstrated the expressivity of random features in kernel methods, allowing kernel machines to scale to large datasets. Similar ideas of random representations were explored in neural networks: for feedforward networks, approaches like \emph{Extreme Learning Machines} \citep{Huang06} and \emph{Random Vector Functional Link} models \citep{Igelnik95} demonstrated that a single trainable linear readout following frozen random layers is sufficient for universal approximation. For sequential data, \emph{reservoir computing} paradigms apply the same strategy to recurrent networks, projecting data into high-dimensional space through a fixed, untrained network (the \emph{reservoir}) and only training a simple readout layer on top of that \citep{Jaeger01,Maass02,teVrugt24}. 

Random and partially random transformers have received growing attention. \citet{Choromanski21} and \citet{Peng21} introduce linear-complexity variants of softmax attention by approximation with random features. \citet{Shen21b} demonstrate one-layer random-weight transformers contain performant subnetworks. Building on reservoir computing ideas, \citet{Shen21a} study the effectiveness of interleaving random layers in a transformer, and \citet{Zhong24} show training only the embeddings of an otherwise random transformer suffices for some nontrivial tasks. \citet{Dong25} discover that random frozen attention performs competitively on language modeling, which they use to analyze the roles of different transformer components. Related to our setting, \citet{Wang25} observed prompt-tuned randomly initialized transformers can effectively memorize finite datasets. 

Our work relies on randomness and mirrors many of the aforementioned ideas. Much like reservoir computing, we also borrow the expressivity of a black box, where the entire frozen untrained transformer network is the reservoir. Instead of training a readout layer, we only control the input to steer the transformer's behavior, and demonstrate that this is expressive enough. Moreover, beyond savings in pretraining cost, we also show how random standard normal weights can lead to smaller soft prompt norms by more effective use of hidden dimension.

\subsection{Summary of Notation}

\begin{table}[h]
\centering
\begin{tabular}{ll}
\hline
\textbf{Notation} & \textbf{Description} \\ \hline
$X$ & Domain (manifold) \\
$d_X$ & Intrinsic manifold dimension \\
$R$ & $\max_{\mathbf{x}\in X}\|\mathbf{x}\|$ \\
$\tau_X$ & Reach of $X$ \\
$\din$ & Input dimension = ambient dimension    \\
$\dh$ & Hidden dimension    \\
$\xemb$ & Embedding vector of $\mathbf{x}$ \\
$\mathcal{G}$ & Transformer class \\ 
$G$ & A transformer \\
$\widetilde{G}$ & The transformer $G$ sans embedding \\
$\alpha$ & H\"{o}lder exponent \\
$L$ & H\"{o}lder constant \\
$S_f$ & $\max_{\mathbf{x}\in X}|f(\mathbf{x})|$ \\
$S_\mathrm{noise}$ & Output noise bound $|\epsilon_i|\leq S_\mathrm{noise}$ \\
$S$ & Combined bound $S_f+S_\mathrm{noise}$ \\
$h$ & Kernel bandwidth \\ 
$\beta$ & Logit shift \\
$n$ & Number of support points = prompt length \\ 
\hline
\end{tabular}
\end{table}

We use $\|\cdot\|$ to denote the $\ell^2$-norm for vectors and the spectral norm for matrices. Boldface letters are reserved for vectors, lowercase regular letters are used for scalars, and matrices are denoted by uppercase regular letters. Throughout, we use the terms \emph{weights} and \emph{parameters} for the entries of the transformer's maps.

\newpage
\section{Preliminaries}

\subsection{Architecture}\label{sec:architecture}
We consider the following hypothesis class ($\mathcal{G}$) of transformer networks: single-layer softmax-attention networks with fixed embedding and decoding maps. We define each component individually below. 

Let $\din$ be the input dimension. The embedding maps each input token into a high-dimensional space of hidden dimension $\dh$ via a homogeneous embedding with additional zero padding, arranged as a row vector. In the deterministic setting (\cref{sec:deterministic}), we need only that $\dh\geq\din+2$; in the later analysis on random versions (\cref{sec:random}) we will assume $\dh\gg\din+2$. 

In this setting, our (hard) prompt will only consist of a single token $\mathbf{x}\in\R^{\din}$, the query point we are predicting the function at, and we denote its embedding as
\begin{equation}\label{eq:homogeneousembedding}
    E(\mathbf{x}) = \xemb  \coloneqq [\mathbf{x}^T,1,0,\ldots,0]\in\R^{1\times \dh}.
\end{equation}

Softmax (self-)attention $A$ acting on a matrix $H\in\R^{(n+1)\times \dh}$ is defined as 
\begin{equation}\label{eq:attn}
    A(H)=\operatorname{softmax}\left(\frac{HW_Q(HW_K)^T}{\sqrt{\dh}}\right)HW_V
\end{equation}
where softmax acts row-wise, with the query, key, and value matrices $W_Q, W_K, W_V\in\R^{\dh\times\dh}$.\footnote{We refer to these as the \emph{attention parameters} to differentiate them from \emph{attention weights}, the standard term for the post-softmax coefficients.} Note that $A$ outputs a matrix of the same size. We only consider single-head attention. 

Since we are only interested in the transformer's prediction for the query $\mathbf{x}$, the decoder $D$ only looks at the last row, where it reads out the bottom-right entry.

The transformer hypothesis class is defined as all compositions of these three components:
\begin{equation}
    \mathcal{G} = \left\{G=D\circ A\circ E \mid W_Q,W_K,W_V\in\R^{\dh\times\dh} \right\}.
\end{equation}
Notice the embedding and decoding are fixed in the current setup; in \cref{sec:randomembdec}, we consider models where these components are also variable.

\subsection{Problem Setup}
We are interested in deriving a universal approximation-type theorem for prompting. Informally, given a function $f:X\to \R$, we wish to find a soft prompt which makes the transformer output approximate $f(\mathbf{x})$ for any given query $\mathbf{x}\in X$.

In more detail, let $G=D\circ A\circ E$ be a transformer network in the aforementioned class $\mathcal{G}$. We denote by $\widetilde{G}=D\circ A$, the part of the network after embedding. We prepend a soft prompt $P=[\mathbf{p}_1, \mathbf{p}_2, \ldots, \mathbf{p}_n]\in\R^{\dh\times n}$ to the embedded $\mathbf{x}$ as
\begin{equation}\label{eq:softpromptedinput}
    H_{P,\mathbf{x}}=\begin{bmatrix}
        \mathbf{p}_1^T \\
        \mathbf{p}_2^T \\
        \vdots \\ 
        \mathbf{p}_n^T \\
        \xemb 
    \end{bmatrix} \in\R^{(n+1)\times \dh}.
\end{equation} 

Our objective is: 
\begin{eqnarray*}
    \text{find a soft prompt} \qquad P=[\mathbf{p}_1, \mathbf{p}_2, \ldots, \mathbf{p}_n]\in\R^{\dh\times n} \nonumber \\
    \text{such that} \qquad \widetilde{G}(H_{P,\mathbf{x}})\approx f(\mathbf{x}) \quad \text{for all}\quad  \mathbf{x}\in X.
\end{eqnarray*} 

The soft prompt should depend on the function being regressed (in particular, the noisy input-output samples of it), but be independent of the specific query.

\subsection{Outline of Proof Strategy}

Let us expand on the informal description of our approach, starting by recalling regression with the \emph{Nadaraya-Watson (NW) kernel estimator}. We assume we have an empirical description of $f$ as noisy input-output pairs: using an unnormalized Gaussian kernel with bandwidth $h>0$, with $n$ support points $\{\mathbf{x}_i\in X\}_{i=1}^n$ and corresponding target values $\{y_i\coloneqq f(\mathbf{x}_i)+\epsilon_i\}_{i=1}^n$, the noisy empirical NW kernel estimator is defined on an $\mathbf{x}\in X$ as 
\begin{equation}\label{eq:kernelformula}
    \widehat{\mathcal{K}}_h(\mathbf{x})=\frac{\sum_{i=1}^n \exp\left(-\frac{\|\mathbf{x}-\mathbf{x}_i\|^2}{2h^2}\right)y_i}{\sum_{i=1}^n \exp\left(-\frac{\|\mathbf{x}-\mathbf{x}_i\|^2}{2h^2}\right)}.
\end{equation}

Given a fixed transformer network in $\mathcal{G}$, we expand the transformer output by definition, assuming an input of the form \eqref{eq:softpromptedinput}. Attention is already a kernel, and we explicitly match, term-by-term, the attention logits with the Gaussian kernel exponents from \eqref{eq:kernelformula}, generating systems of linear equations to determine each $\mathbf{p}_i$. The solvability of these linear systems is guaranteed by a mild condition on the weights, which we show, in later sections, is satisfied almost surely by Gaussian initializations. 

There is an additional term remaining in the transformer output, corresponding to the self-attention contribution of the query token with itself, which does not appear in \eqref{eq:kernelformula}. We suppress this term with \emph{logit shifting}, introducing an additional variable $\beta\geq0$ that inflates the relevant attention logits. This is folded into the linear systems, which we then solve using the minimum-norm solutions. With this, we have shown that given a fixed transformer and any fixed set of kernel support points, we may construct a soft prompt such that the transformer output approximates the output of the empirical kernel estimator.

Leveraging existing results in approximation theory for kernel regression, we extend our construction to complete and constructive universal approximation results. In particular, we bound the contribution of noise to the estimator, then use bounds for the empirical-integral gap of the kernel estimator and the inherent kernel regression bias.

We require the following standard assumptions to derive explicit bounds. For the domain, we restrict ourselves to a compact manifold.

\begin{assumption}[Domain]\label{ass:domain}
    The domain $X$ is a compact $d_X$-dimensional Riemannian manifold ($d_X\geq1$) with positive reach $\tau_X>0$, isometrically embedded in $\R^{\din}$, satisfying $\|\mathbf{x}\|\leq R$ for all $\mathbf{x}\in X$. 
\end{assumption}

The function class is restricted to H\"{o}lder functions on that manifold. 

\begin{assumption}[Function class]\label{ass:function}
    The function class is restricted to $f\in C^{\alpha}_L(X)$, consisting of $\alpha$-H\"{o}lder functions $X\to\R$ with H\"{o}lder constant $L$, where $0< \alpha\leq 1$. 
\end{assumption}

For a fixed such $f$, let $S_f\coloneqq\max_{\mathbf{x}\in X}|f(\mathbf{x})|$. Finally, we assume bounded noise.

\begin{assumption}[Label noise]\label{ass:noise}
    Labels are observed with noise: $y_i=f(\mathbf{x}_i)+\epsilon_i$, where $\epsilon_1, \ldots, \epsilon_n$ are independent, mean-zero, and satisfy $|\epsilon_i|\leq S_{\mathrm{noise}}$ almost surely, independent of the support points $\{\mathbf{x}_i\}_{i=1}^n$ and the query $x$.  
\end{assumption}

Denote by $S$ the combined upper bound $S\coloneqq S_f+S_\mathrm{noise}$. 

\section{Deterministic Weights}\label{sec:deterministic}

\subsection{Prompting can Approximate the NW Kernel Estimator}

We first show, in this subsection, that with an appropriate prompt, the transformer output can approximate the NW kernel estimator \eqref{eq:kernelformula} with any fixed set of support points. In the next subsection, we extend this to a full universal approximation theorem using standard approximation theory results. 

Our approach and subsequent proof will suggest a sufficient condition on the attention parameters, which we state here. Let $W_Q, W_K, W_V$ denote the parameters for the attention layer $A$ of a transformer $G\in\mathcal{G}$. Denote by $W=\frac{1}{\sqrt{\dh}}W_QW_K^T$. Let $W_{1:\din}$ denote the first $\din$ rows of $W$, and $W_{\din+1}$ be the $(\din+1)$-th row. Let $\mathbf{w}_V$ denote the last column of $W_V$. Define the matrix $\widetilde{W}\in\R^{(\din+2)\times \dh}$ as 
\begin{equation}
    \widetilde{W}=\begin{bmatrix} W_{1:\din} \\ W_{\din+1} \\ \mathbf{w}_V^T \end{bmatrix}.
\end{equation}

\begin{condition}[Solvability of prompt]\label{ass:parameters}
    The attention parameters of $G$ are such that $\widetilde{W}$ has full row rank. 
\end{condition}

With this, we state our first theorem.  

\begin{theorem}[Prompting can approximate NW kernel estimator]\label{thm:promptingkernel}
    Let $X$ satisfy \cref{ass:domain} and let $f:X\to\R$ satisfy \cref{ass:function}. Let $n\geq1$ and fix support points $\{\mathbf{x}_i\in X\}_{i=1}^n$ with labels satisfying \cref{ass:noise}; let $\widehat{\mathcal{K}}_h(\mathbf{x})$ denote the Nadaraya-Watson kernel estimator evaluated at $\mathbf{x}$, constructed using these support points and bandwidth $h>0$. Let $G\in\mathcal{G}$ be a transformer network whose attention parameters satisfy \cref{ass:parameters}. For any logit shift $\beta\geq0$, there exists a corresponding soft prompt $P^*$ of length $n$ such that
    \begin{equation}
        \sup_{\mathbf{x}\in X}\left( \widetilde{G}(H_{P^*,\mathbf{x}}) - \widehat{\mathcal{K}}_h(\mathbf{x}) \right)^2 = O\left(\frac{1}{n^2e^{2\beta}}\cdot\exp\left(\frac{3R^2}{h^2}\right)\right),
    \end{equation}
    where $O(\cdot)$ depends on $\|W\|, \|\mathbf{w}_V\|, R$, and $S$. Furthermore, the soft prompt tokens satisfy 
    \begin{equation}
        \|\mathbf{p}_i^*\|=O(\beta+h^{-2})
    \end{equation} 
    for all $1\leq i\leq n$, where $O(\cdot)$ hides a dependence on $\sigma_{\min}(\widetilde{W})^{-1},R,$ and $S$.
    \end{theorem}

    \begin{proof}
    Expanding the definition of attention \eqref{eq:attn}, we first find that the last row of $A(H_{P,\mathbf{x}})$ is
    \begin{equation}\label{eq:attnoutput}
    A(H_{P,\mathbf{x}})_{n+1} = \frac{\sum_{i=1}^n \exp\left(\frac{1}{\sqrt{\dh}}\xemb  W_Q W_K^T \mathbf{p}_i\right) \mathbf{p}_i^T W_V + \exp\left(\frac{1}{\sqrt{\dh}}\xemb  W_Q W_K^T \xemb ^T\right) \xemb  W_V}{\sum_{i=1}^n \exp\left(\frac{1}{\sqrt{\dh}}\xemb  W_Q W_K^T \mathbf{p}_i\right) + \exp\left(\frac{1}{\sqrt{\dh}}\xemb  W_Q W_K^T \xemb ^T\right)}.
    \end{equation}
    
    To get the whole $\widetilde{G}(H_{P,\mathbf{x}})$ we decode the attention output, reading off the last coordinate. Mathematically, this corresponds to right multiplying the above output by the $\dh$-th standard basis vector $\mathbf{e}_{\dh}=[0,\ldots, 0, 1]^T$. 
    
    Note that the attention contribution of $\xemb $ with itself is independent of $\mathbf{p}_i$. Let us temporarily ignore this part, leaving the following output after decoding:
    \begin{equation}\label{eq:simplifiedattnoutput}
    \frac{\sum_{i=1}^n \exp\left(\frac{1}{\sqrt{\dh}}\xemb  W_Q W_K^T \mathbf{p}_i\right) \mathbf{p}_i^T W_V \mathbf{e}_{\dh}}{\sum_{i=1}^n \exp\left(\frac{1}{\sqrt{\dh}}\xemb  W_Q W_K^T \mathbf{p}_i\right)}.
    \end{equation}
    
    We recall the kernel estimator \eqref{eq:kernelformula} here:
    \begin{equation}
    \widehat{\mathcal{K}}_h(\mathbf{x})=\frac{\sum_{i=1}^n \exp\left(-\frac{\|\mathbf{x}-\mathbf{x}_i\|^2}{2h^2}\right)y_i}{\sum_{i=1}^n \exp\left(-\frac{\|\mathbf{x}-\mathbf{x}_i\|^2}{2h^2}\right)}.
    \end{equation}

    We rewrite this in a more convenient form. For each $i$, the exponents expand as 
    \begin{equation}
    -\frac{\|\mathbf{x}-\mathbf{x}_i\|^2}{2h^2} = -\frac{\|\mathbf{x}\|^2}{2h^2} + \frac{\mathbf{x}^T \mathbf{x}_i}{h^2} - \frac{\|\mathbf{x}_i\|^2}{2h^2},
    \end{equation}
    whereupon we may split off the first part from the exponentiation:
    \begin{equation}
    \exp\left(-\frac{\|\mathbf{x}-\mathbf{x}_i\|^2}{2h^2}\right) = \exp\left(-\frac{\|\mathbf{x}\|^2}{2h^2}\right) \exp\left(\frac{\mathbf{x}^T \mathbf{x}_i}{h^2}-\frac{\|\mathbf{x}_i\|^2}{2h^2}\right).
    \end{equation}
    
    The first term is independent of the index $i$, thus may be factored out of the numerator and denominator and canceled. This leaves the mathematically equivalent formula of the kernel estimator:
    \begin{equation}\label{eq:newkernelformula}
    \widehat{\mathcal{K}}_h(\mathbf{x})=\frac{\sum_{i=1}^n \exp\left(\frac{\mathbf{x}^T \mathbf{x}_i}{h^2}-\frac{\|\mathbf{x}_i\|^2}{2h^2}\right)y_i}{\sum_{i=1}^n \exp\left(\frac{\mathbf{x}^T \mathbf{x}_i}{h^2}-\frac{\|\mathbf{x}_i\|^2}{2h^2}\right)},
    \end{equation}
    where $y_i=f(\mathbf{x}_i)+\epsilon_i$. 
    
    Comparing the attention output in \eqref{eq:simplifiedattnoutput} to the above formula, we see that the soft prompt should satisfy two properties. First, we should have $\mathbf{p}_i^T W_V \mathbf{e}_{\dh} = y_i$ for all $i$, i.e. the projected values align with the true target values. $W_V \mathbf{e}_{\dh}$ is the last column of $W_V$, which we denote $\mathbf{w}_V$. We rewrite the condition as $\mathbf{w}_V^T \mathbf{p}_i=y_i$.
    
    We also want to match the kernel part, which involves matching the attention logits with the exponents for all $i$, i.e. 
    \begin{equation}\label{eq:kernelmatching}
    \frac{1}{\sqrt{\dh}}\xemb  W_Q W_K^T \mathbf{p}_i =  \frac{\mathbf{x}^T \mathbf{x}_i}{h^2} - \frac{\|\mathbf{x}_i\|^2}{2h^2}.
    \end{equation}
    
    Denote by $W=\frac{1}{\sqrt{\dh}}W_QW_K^T$. Note that due to the zero padding in $\mathbf{x}_\mathrm{emb}$, we effectively only consider the first $\din+1$ rows of $W$. We denote the first $\din$ rows as $W_{1:\din}$, and the $(\din+1)$-th row as $W_{\din+1}$. The attention logit rewrites as 
    \begin{equation}
        \xemb W\mathbf{p}_i = \mathbf{x}^TW_{1:\din}\mathbf{p}_i+1\cdot W_{\din+1}\mathbf{p}_i.
    \end{equation}
        
    We match respective parts with the RHS of \eqref{eq:kernelmatching}:
    \begin{equation}\mathbf{x}^T (W_{1:\din} \mathbf{p}_i) = \frac{\mathbf{x}^T \mathbf{x}_i}{h^2},
    \end{equation}
    for which it suffices to enforce 
    \begin{equation}
        W_{1:\din} \mathbf{p}_i=\frac{\mathbf{x}_i}{h^2}
    \end{equation}
    and
    \begin{equation}\label{eq:biasnormmatching}
        W_{\din+1}\mathbf{p}_i=-\frac{\|\mathbf{x}_i\|^2}{2h^2}.
    \end{equation}

    This is the purpose of the homogeneous coordinate: the bias allows us to absorb scalar shifts. It also allows us to handle the contribution of the query's attention with itself, which we have hitherto ignored, using logit shifting—recalibrating attention by artificially inflating certain logits. Adding decoding and rewriting equation \eqref{eq:attnoutput} with our new notation, we get
    \begin{equation}
        \widetilde{G}(H_{P,\mathbf{x}}) = \frac{\sum_{i=1}^n\exp(\xemb W \mathbf{p}_i)\mathbf{p}_i^T\mathbf{w}_V+\exp(\xemb W\xemb^T)\xemb\mathbf{w}_V}{\sum_{i=1}^n\exp(\xemb W \mathbf{p}_i)+\exp(\xemb W\xemb^T)}.
    \end{equation}
    
    Notice the effect of adding a constant $\beta\geq0$ to the logits which involve the soft prompt tokens:
    \begin{align}
        &\frac{\sum_{i=1}^n\exp(\xemb W \mathbf{p}_i+\beta)\mathbf{p}_i^T\mathbf{w}_V+\exp(\xemb W\xemb^T)\xemb\mathbf{w}_V}{\sum_{i=1}^n\exp(\xemb W \mathbf{p}_i+\beta)+\exp(\xemb W\xemb^T)} \\
        &= \frac{e^\beta\sum_{i=1}^n\exp(\xemb W \mathbf{p}_i)\mathbf{p}_i^T\mathbf{w}_V+\exp(\xemb W\xemb^T)\xemb\mathbf{w}_V}{e^\beta\sum_{i=1}^n\exp(\xemb W \mathbf{p}_i)+\exp(\xemb W\xemb^T)} \\
        &= \frac{\sum_{i=1}^n\exp(\xemb W \mathbf{p}_i)\mathbf{p}_i^T\mathbf{w}_V+\frac{1}{e^\beta}\exp(\xemb W\xemb^T)\xemb\mathbf{w}_V}{\sum_{i=1}^n\exp(\xemb W \mathbf{p}_i)+\frac{1}{e^\beta}\exp(\xemb W\xemb^T)}.
    \end{align}

    For large $\beta$, the attention contribution of the query with itself is negligible. We may realize this constant shift in the bias term by modifying condition \eqref{eq:biasnormmatching} to

    \begin{equation}
        W_{\din+1}\mathbf{p}_i=-\frac{\|\mathbf{x}_i\|^2}{2h^2}+\beta.
    \end{equation}
    
    This comes at the cost of a larger (in magnitude) soft prompt; we will formally bound this later. Adding the additional value matching condition $\mathbf{w}_V^T \mathbf{p}_i=y_i$ from before, for each $\mathbf{p}_i$ we are left with a linear system $\widetilde{W}\mathbf{p}_i=\mathbf{v}_i$, where 
    \begin{equation}\label{eq:wandv}
        \widetilde{W}=\begin{bmatrix} W_{1:\din} \\ W_{\din+1} \\ \mathbf{w}_V^T \end{bmatrix} \qquad \text{ and } \qquad \mathbf{v}_i=\begin{bmatrix} \frac{\mathbf{x}_i}{h^2} \\ -\frac{\|\mathbf{x}_i\|^2}{2h^2} + \beta \\ y_i \end{bmatrix}.
    \end{equation}

    By \cref{ass:parameters}, the matrix $\widetilde{W}$ has full row rank, thus the system has at least one solution for any target vector $\mathbf{v}_i$. We construct the minimum-norm solution by the Moore-Penrose pseudoinverse $\widetilde{W}^\dagger$:

    \begin{equation}
        \mathbf{p}_i^* = \widetilde{W}^\dagger\mathbf{v}_i \coloneqq \widetilde{W}^T (\widetilde{W} \widetilde{W}^T)^{-1} \mathbf{v}_i.
    \end{equation}
    
    Assembling the soft prompt tokens, we get a soft prompt $P^*$ such that
    \begin{equation}
        \widetilde{G}(H_{P^*,\mathbf{x}}) = \frac{\sum_{i=1}^n\exp\left(\frac{\mathbf{x}^T \mathbf{x}_i}{h^2}-\frac{\|\mathbf{x}_i\|^2}{2h^2}\right)y_i+\frac{1}{e^\beta}\exp(\xemb W\xemb^T)\xemb\mathbf{w}_V}{\sum_{i=1}^n\exp\left(\frac{\mathbf{x}^T \mathbf{x}_i}{h^2}-\frac{\|\mathbf{x}_i\|^2}{2h^2}\right)+\frac{1}{e^\beta}\exp(\xemb W\xemb^T)}.
    \end{equation}

    It remains to uniformly bound the difference $\left| \widetilde{G}(H_{P^*,\mathbf{x}}) - \widehat{\mathcal{K}}_h(\mathbf{x}) \right|$. Factoring out the denominator, the difference simplifies to
    \begin{align}\label{eq:outputkerneldifference}
        \left| \widetilde{G}(H_{P^*,\mathbf{x}}) - \widehat{\mathcal{K}}_h(\mathbf{x}) \right| &= 
        \frac{\left| \frac{1}{e^\beta}\exp(\xemb W\xemb^T)\xemb\mathbf{w}_V - \widehat{\mathcal{K}}_h(\mathbf{x}) \frac{1}{e^\beta}\exp(\xemb W\xemb^T) \right|}
        {\sum_{i=1}^n\exp\left(\frac{\mathbf{x}^T \mathbf{x}_i}{h^2}-\frac{\|\mathbf{x}_i\|^2}{2h^2}\right) + \frac{1}{e^\beta}\exp(\xemb W\xemb^T)} \nonumber 
        \\ &\leq \frac{\frac{1}{e^\beta}\left(\left| \exp(\xemb W\xemb^T)\right| |\xemb\mathbf{w}_V| + \left|\widehat{\mathcal{K}}_h(\mathbf{x})\right| \left| \exp(\xemb W\xemb^T) \right|\right)}
        {\sum_{i=1}^n\exp\left(\frac{\mathbf{x}^T \mathbf{x}_i}{h^2}-\frac{\|\mathbf{x}_i\|^2}{2h^2}\right) + \frac{1}{e^\beta}\exp(\xemb W\xemb^T)} \nonumber 
        \\ &\leq \frac{1}{e^\beta}\cdot\frac{\left| \exp(\xemb W\xemb^T)\right|\left(|\xemb\mathbf{w}_V| + \left|\widehat{\mathcal{K}}_h(\mathbf{x})\right|\right)}
        {\sum_{i=1}^n\exp\left(\frac{\mathbf{x}^T \mathbf{x}_i}{h^2}-\frac{\|\mathbf{x}_i\|^2}{2h^2}\right)}
    \end{align}
    where the first inequality comes from the triangle inequality, and the second from discarding the positive term in the denominator. 

    We establish bounds on each component. Recall, in particular, \cref{ass:domain,ass:function,ass:noise}. It follows from the norm bound on the query and support points that
    \begin{equation}
        \frac{\mathbf{x}^T \mathbf{x}_i}{h^2}-\frac{\|\mathbf{x}_i\|^2}{2h^2} \geq -\frac{R^2}{h^2}-\frac{R^2}{2h^2}=-\frac{3R^2}{2h^2},
    \end{equation}
    from which we may conclude the following bound on the denominator of the kernel estimator:
    \begin{equation}
        \sum_{i=1}^n\exp\left(\frac{\mathbf{x}^T \mathbf{x}_i}{h^2}-\frac{\|\mathbf{x}_i\|^2}{2h^2}\right) \geq n\exp\left(-\frac{3R^2}{2h^2}\right)>0.
    \end{equation}

    The homogeneous embedding of $\mathbf{x}$ specified in \eqref{eq:homogeneousembedding} satisfies $\|\xemb\| \leq \sqrt{R^2+1}$. We have that 
    \begin{equation}
        |\xemb W\xemb^T| \leq \|\xemb\|^2\|W\| \leq (R^2+1)\cdot\|W\|
    \end{equation}
    and 
    \begin{equation}
        |\xemb\mathbf{w}_V| \leq \|\xemb\|\|\mathbf{w}_V\| \leq \sqrt{R^2+1}\cdot\|\mathbf{w}_V\|.
    \end{equation}
    Exponentiating everywhere gives the corresponding bounds. 

    Finally, the kernel estimator is a convex combination of target values $y_i$, thus its absolute value is bounded above by 
    \begin{equation}
        \left|\widehat{\mathcal{K}}_h(\mathbf{x})\right| \leq \max_{1\leq i\leq n}|y_i|\leq S.
    \end{equation}

    Plugging everything back into \eqref{eq:outputkerneldifference}, we arrive at the final bound 
    \begin{equation}\label{eq:expandedapproxerror}
        \sup_{\mathbf{x}\in X}\left|\widetilde{G}(H_{P^*,\mathbf{x}}) - \widehat{\mathcal{K}}_h(\mathbf{x})\right| \leq \frac{1}{ne^\beta}\exp\left(\frac{3R^2}{2h^2}\right)\exp\left((R^2+1)\|W\|\right)\left(\sqrt{R^2+1}\|\mathbf{w}_V\| + S\right).
    \end{equation}
    Squaring both sides and hiding constants gives our desired result:
    \begin{equation}
        \sup_{\mathbf{x}\in X}\left(\widetilde{G}(H_{P^*,\mathbf{x}}) - \widehat{\mathcal{K}}_h(\mathbf{x})\right)^2 = O\left(\frac{1}{n^2e^{2\beta}}\cdot\exp\left(\frac{3R^2}{h^2}\right)\right).
    \end{equation}
    
    Asymptotically, we have quadratic decay in terms of soft prompt length $n$, exponential decay in terms of the logit shift $\beta$, and exponential growth in inverse kernel bandwidth $h$. 

    To bound the norm of the constructed soft prompt tokens, we first derive a uniform bound on the $\|\mathbf{v}_i\|$, which was defined in \eqref{eq:wandv}. By \cref{ass:domain}, $\|\mathbf{x}_i\|\leq R$, so $\|\mathbf{x}_i/h^2\|\leq R/h^2$. Similarly, taking absolute value of the second term yields
    \begin{equation}
        \left|-\frac{\|\mathbf{x}_i\|^2}{2h^2}+\beta\right| \leq \frac{R^2}{2h^2}+\beta,
    \end{equation}
    since $\beta, h$, and $R$ are nonnegative. Lastly, by definition, $|y_i|\leq S$. Together, we have
    \begin{equation}
        \|\mathbf{v}_i\| \leq \sqrt{\frac{R^2}{h^4} + \left(\beta+\frac{R^2}{2h^2}\right)^2 + S^2}
    \end{equation}
    for any $1\leq i\leq n$. 

    Since $\widetilde{W}\in\R^{(\din+2)\times\dh}$ is wide with full row rank (\cref{ass:parameters}), its smallest singular value $\sigma_{\min}(\widetilde{W})$ is guaranteed to be strictly positive. The Moore-Penrose pseudoinverse has spectral norm equal to the reciprocal of this smallest singular value:
    \begin{equation}
        \|\widetilde{W}^\dagger\| = \frac{1}{\sigma_{\min}(\widetilde{W})}.
    \end{equation}
    
    Submultiplicativity then yields our total bound: for all $1\leq i\leq n$, 
    \begin{equation}
        \|\mathbf{p}_i^*\| \leq \|\widetilde{W}^\dagger\|\|\mathbf{v}_i\| \leq \frac{1}{\sigma_{\min}(\widetilde{W})}\sqrt{\frac{R^2}{h^4} + \left(\beta+\frac{R^2}{2h^2}\right)^2 + S^2}=O(\beta+h^{-2}).
    \end{equation}
\end{proof}

\begin{remark}
    Logit shifting allows us to exponentially suppress the unwanted query self-attention term, but it is not strictly necessary. If we set the shift to $\beta=0$, the squared error still decays at a quadratic rate in prompt length, though this alone is insufficient for our forthcoming universal approximation theorems. Alternatively, this issue could be handled via prompt duplication (concatenating exact copies of the soft prompt to inflate their mass in the kernel), or bypassed entirely by modifying the architecture to use a masked/cross-attention that explicitly prevents query self-attention, allowing for exact reconstruction of the kernel estimator. 

    Our approach relies on logit shifting for two reasons: it achieves an exponential rate under a standard, unmodified self-attention regime, and it aligns with the real-world computational preference for scaling token magnitude over increasing prompt length. The exponential decay in $\beta$ will be important to balance the bandwidth sensitivity for \cref{thm:univapprox} in the next subsection, and also some random constants in the analogous theorems in \cref{sec:random}.
\end{remark}

\begin{remark}
    \cref{ass:parameters} is stronger than is necessary for the existence of a sufficient soft prompt from this procedure, which only requires $\mathbf{v}_i$ to be in the column space of $\widetilde{W}$. We use it to also track the prompt norm of the construction, and have a cleaner sample-independent condition. 
\end{remark}

\subsection{Universal Approximation}

Using standard techniques and results from approximation theory, the results of the previous subsection can be extended to derive complete universal approximation error rates for regression. Mirroring the in-context learning setup, we study regression with sampled support points (as opposed to fixed designs) and a sampled query. 

\begin{assumption}[Distribution of support points and query]\label{ass:supportpoints}
    Support points $\mathbf{x}_1,\mathbf{x}_2,\ldots,\mathbf{x}_n$ and the query $\mathbf{x}$ are drawn iid from the uniform distribution on $X$, denoted $\mathcal{U}(X)$. 
\end{assumption}

We decompose the total approximation error by introducing some more intermediate terms (beyond \cref{thm:promptingkernel}), which we define here. First, to consider the error induced by noise, define the \emph{noiseless} empirical kernel estimator analogously as \eqref{eq:kernelformula}: 
\begin{equation}
    \overline{\mathcal{K}}_h(\mathbf{x}) = \frac{\sum_{i=1}^n \exp\left(-\frac{\|\mathbf{x}-\mathbf{x}_i\|^2}{2h^2}\right)f(\mathbf{x}_i)}{\sum_{i=1}^n \exp\left(-\frac{\|\mathbf{x}-\mathbf{x}_i\|^2}{2h^2}\right)}.
\end{equation}

We may bound the difference between the noisy and noiseless kernel estimators with the following lemma.

\begin{lemma}[Noise term of kernel estimator]\label{lem:kernelnoise}
    Suppose $X$ satisfies \cref{ass:domain}, the support points and query satisfy \cref{ass:supportpoints}, and the labels satisfy \cref{ass:noise}. Then for any $\mathbf{x}\in X$ and $0<h\leq \tau_X/2$, 
    \begin{equation}
        \E_{\mathbf{x}_i,\epsilon_i} \left(\widehat{\mathcal{K}}_h(\mathbf{x})-\overline{\mathcal{K}}_h(\mathbf{x})\right)^2=O\left(\frac{S_\mathrm{noise}^2}{nh^{d_X}}\right),
    \end{equation}
    where $O(\cdot)$ hides dependence on $d_X$ and $\tau_X$. 
\end{lemma}

\begin{proof}
    Expanding out definitions, notice that the core of the estimator involving $f(\mathbf{x}_i)$ cancels out, and we are only left with the noise contribution:

    \begin{equation}
        \widehat{\mathcal{K}}_h(\mathbf{x})-\overline{\mathcal{K}}_h(\mathbf{x}) = \frac{\sum_{i=1}^n \exp\left(-\frac{\|\mathbf{x}-\mathbf{x}_i\|^2}{2h^2}\right)\epsilon_i}{\sum_{i=1}^n \exp\left(-\frac{\|\mathbf{x}-\mathbf{x}_i\|^2}{2h^2}\right)}.
    \end{equation}

    Squaring and conditioning on the support points and query, note that \cref{ass:noise} implies the cross-terms $\E[\epsilon_i\epsilon_j]$ vanish, leaving
    \begin{align}
        \E \left[\left(\frac{\sum_{i=1}^n \exp\left(-\frac{\|\mathbf{x}-\mathbf{x}_i\|^2}{2h^2}\right)\epsilon_i}{\sum_{i=1}^n \exp\left(-\frac{\|\mathbf{x}-\mathbf{x}_i\|^2}{2h^2}\right)}\right)^2 \Bigg| \{\mathbf{x}_i\}_{i=1}^n, \mathbf{x}\right] 
        &= \frac{\sum_{i=1}^n \exp\left(-\frac{\|\mathbf{x}-\mathbf{x}_i\|^2}{2h^2}\right)^2\E(\epsilon_i^2)}{\left(\sum_{i=1}^n \exp\left(-\frac{\|\mathbf{x}-\mathbf{x}_i\|^2}{2h^2}\right)\right)^2} \\ 
        &\leq S_\mathrm{noise}^2 \cdot\min\left(1, \frac{1}{\sum_{i=1}^n \exp\left(-\frac{\|\mathbf{x}-\mathbf{x}_i\|^2}{2h^2}\right)}\right),
    \end{align}
    by \cref{ass:noise} and since $0<\exp\left(-\frac{\|\mathbf{x}-\mathbf{x}_i\|^2}{2h^2}\right)\leq 1$.
    
    Each support point $\mathbf{x}_j$ such that $\|\mathbf{x}-\mathbf{x}_j\|\leq h$ will contribute at least $e^{-1/2}$ to the denominator; letting $m$ denote the number of such points, the sum can be bounded as 
    \begin{equation}
        \sum_{i=1}^n \exp\left(-\frac{\|\mathbf{x}-\mathbf{x}_i\|^2}{2h^2}\right) \geq e^{-1/2}\cdot m.
    \end{equation}
    We then have
    \begin{equation}
        \min\left(1, \frac{1}{\sum_{i=1}^n \exp\left(-\frac{\|\mathbf{x}-\mathbf{x}_i\|^2}{2h^2}\right)}\right) \leq \min\left(1, \frac{e^{1/2}}{ m}\right) \leq \frac{2e^{1/2}}{1+m}.
    \end{equation}
    By \cref{ass:supportpoints} the support points are iid uniform on $X$ (and independent of $\mathbf{x}$), thus when conditioned on the query, $m$ is binomial with $n$ trials and success probability $\mathbb{P}_{\mathbf{t}\sim\mathcal{U}(X)}(\|\mathbf{x}-\mathbf{t}\|\leq h)$. The first inverse moment identity for binomial variables yields 
    \begin{equation}
        \E\left(\frac{2e^{1/2}}{1+m}\;\big|\;\mathbf{x}\right)\leq \frac{2e^{1/2}}{(n+1)\mathbb{P}_{\mathbf{t}\sim\mathcal{U}(X)}(\|\mathbf{x}-\mathbf{t}\|\leq h)}.
    \end{equation}
    Finally, standard small-ball probability bounds (see, e.g., \citet{Niyogi08}) give that $\mathbb{P}_{\mathbf{t}\sim\mathcal{U}(X)}(\|\mathbf{x}-\mathbf{t}\|\leq h)= \Omega(h^{d_X})$
    for any $\mathbf{x}\in X$ and $0<h\leq \tau_X/2$, where the hidden constant depends only on $d_X$ and $\tau_X$. Combining the above, we may conclude the result:
    \begin{equation}
        \E_{\mathbf{x}_i,\epsilon_i} \left(\widehat{\mathcal{K}}_h(\mathbf{x})-\overline{\mathcal{K}}_h(\mathbf{x})\right)^2=O\left(\frac{S_\mathrm{noise}^2}{nh^{d_X}}\right)
    \end{equation}
    for any $\mathbf{x}\in X$. 
\end{proof}

The remaining two components in the error decomposition are the finite-sample gap and kernel regression bias, both of which involve the noiseless \emph{integral} kernel estimator:
\begin{equation}
    \mathcal{K}_h(\mathbf{x})=\frac{\int_X \exp\left(-\frac{\|\mathbf{x}-\mathbf{t}\|^2}{2h^2}\right)f(\mathbf{t})\,\mathrm{d}\mathbf{t}}{\int_X \exp\left(-\frac{\|\mathbf{x}-\mathbf{t}\|^2}{2h^2}\right)\,\mathrm{d}\mathbf{t}}.
\end{equation}

For these, we recall some results regarding noiseless kernel regression on manifolds, directly adapted from \citet{Shen26}. The proofs are deferred to their work. 

\begin{lemma}[Variance of kernel estimator]\label{lem:kernelvariance}
    Suppose $X$ and $f$ satisfy \cref{ass:domain} and \cref{ass:function} respectively, and suppose we have support points $\{\mathbf{x}_i\}_{i=1}^n$ and a query $\mathbf{x}$ satisfying \cref{ass:supportpoints}. For any $\delta\in(0,1]$, we have that
    \begin{equation}
        \left|\overline{\mathcal{K}}_h(\mathbf{x})-\mathcal{K}_h(\mathbf{x})\right| = O\left(\left(\log\left(h^{-1}\right)\right)^{3d_X/4}\cdot\sqrt{\frac{\log(4/\delta)}{nh^{d_X}}}\right)
    \end{equation}
    with probability at least $1-\delta$, where $h>0$ is the kernel bandwidth. Here, $O(\cdot)$ hides dependence on $d_X,S_f,$ and $\tau_X$, where the dependency on $d_X$ can be $d_X^{d_X/2}$ in the worst case. 
\end{lemma}

\begin{lemma}[Bias of kernel regression on manifold]\label{lem:kernelbias}
    Suppose $X$ and $f$ satisfy \cref{ass:domain} and \cref{ass:function} respectively. Uniformly over $X$, the integral kernel estimator $\mathcal{K}_h$ for $f$ with bandwidth $h>0$ satisfies
    \begin{equation}
        \left|\mathcal{K}_h(\mathbf{x})-f(\mathbf{x})\right| = O(h^\alpha\log(h^{-1})).
    \end{equation}
    The $O(\cdot)$ hides dependence on $d_X,L,S_f,$ and $\tau_X$, with at most exponential dependence on $d_X$. 
\end{lemma}

With these in hand, we state and prove our universal approximation theorem for prompting in the deterministic case, showing that under certain assumptions, any function can be approximated at a specific rate by prompting. Note that by our approach of going through the NW kernel estimator, we also inherit its positive qualities: we recover the minimax-optimal convergence rate for noisy nonparametric regression on manifolds (up to logarithmic factors), which is independent of the ambient dimension $\din$ (up to constants, through $R$). 

\begin{theorem}[Universal approximation for prompting]\label{thm:univapprox}
    Let $X$ satisfy \cref{ass:domain} and $G\in\mathcal{G}$ be a transformer network whose attention parameters satisfy \cref{ass:parameters}. Suppose the support points and query satisfy \cref{ass:supportpoints}, and the labels satisfy \cref{ass:noise}. For any $f$ satisfying \cref{ass:function}, there exists a map assigning to each realization of samples $\{\mathbf{x}_i,y_i=f(\mathbf{x}_i)+\epsilon_i\}_{i=1}^n$ a soft prompt $P^*$ of length $n$, independent of the query $\mathbf{x}$, such that
    \begin{equation}
        \E_{\mathbf{x}_i,\epsilon_i,\mathbf{x}} (\widetilde{G}(H_{P^*,\mathbf{x}})-f(\mathbf{x}))^2 = O\left(n^{-\frac{2\alpha}{2\alpha+d_X}}(\log n)^{1+\frac{3d_X}{2}}\right).
    \end{equation}
    Moreover, uniformly over all realizations of samples, the tokens of $P^*$ have norm bounded as
    \begin{equation}
        \|\mathbf{p}^*_i\| = O(n^{\frac{2}{2\alpha+d_X}}).
    \end{equation}
\end{theorem}

\begin{proof}
    Fix a query $\mathbf{x}\in X$ and a set of support points $\{\mathbf{x}_i\}_{i=1}^n$. Inserting intermediate terms leads to the following decomposition:
    \begin{align*}
        &\left(\widetilde{G}(H_{P^*,\mathbf{x}})-f(\mathbf{x})\right)^2 \\
        &\quad\leq 4\underbrace{\left(\widetilde{G}(H_{P^*,\mathbf{x}})-\widehat{\mathcal{K}}_h(\mathbf{x})\right)^2}_{\mathrm{I}}
        + 4\underbrace{\left(\widehat{\mathcal{K}}_h(\mathbf{x})-\overline{\mathcal{K}}_h(\mathbf{x})\right)^2}_{\mathrm{II}} + 4\underbrace{\left(\overline{\mathcal{K}}_h(\mathbf{x})-\mathcal{K}_h(\mathbf{x})\right)^2}_{\mathrm{III}}
        + 4\underbrace{\left(\mathcal{K}_h(\mathbf{x})-f(\mathbf{x})\right)^2}_{\mathrm{IV}}.
    \end{align*}

    We bound the four terms separately. The first term $\mathrm{I}$ is essentially \cref{thm:promptingkernel}, which was proved by the explicit construction of an appropriate soft prompt. Terms $\mathrm{II}$, $\mathrm{III}$, and $\mathrm{IV}$ are bounded by \cref{lem:kernelnoise,,lem:kernelvariance,,lem:kernelbias} respectively.

    To recover minimax-optimal rates, we choose the kernel bandwidth and logit shift as
    \begin{equation}
        h=n^{-\frac{1}{2\alpha+d_X}} \qquad \text{ and } \qquad \beta=\frac{3R^2}{2}\cdot n^{\frac{2}{2\alpha+d_X}}
    \end{equation}
    respectively. Since we are proving an asymptotic (in $n$) bound, we may assume $n>(2/\tau_X)^{2\alpha+d_X}$, so that $0<h<\min(1/2,\tau_X/2)$. 
    
    Recall from (the proof of) \cref{thm:promptingkernel}, we have
    \begin{equation}
        \sup_{\mathbf{x}\in X}\left(\widetilde{G}(H_{P^*,\mathbf{x}}) - \widehat{\mathcal{K}}_h(\mathbf{x})\right)^2 = O\left(\frac{1}{n^2e^{2\beta}}\cdot\exp\left(\frac{3R^2}{h^2}\right)\right).
    \end{equation}
    This result holds independently of the specific choice of support points, and the supremum over the query implies the same bound holds in expectation. Plugging in our choice of $h$ and $\beta$, the exponents cancel exactly, leaving
    \begin{equation}\label{eq:Ibound}
        \E_{\mathbf{x}_i,\mathbf{x}\sim\mathcal{U}(X)}\left(\widetilde{G}(H_{P^*,\mathbf{x}}) - \widehat{\mathcal{K}}_h(\mathbf{x})\right)^2 = O(n^{-2}).
    \end{equation}
    
    By \cref{lem:kernelnoise}, term $\mathrm{II}$ is directly bounded as
    \begin{equation}
        \E_{\mathbf{x}_i,\epsilon_i} \left(\widehat{\mathcal{K}}_h(\mathbf{x})-\overline{\mathcal{K}}_h(\mathbf{x})\right)^2=O\left(\frac{S_\mathrm{noise}^2}{nh^{d_X}}\right).
    \end{equation}
    Substituting the choice for $h$ and hiding $S_\mathrm{noise}$, 
    \begin{equation}\label{eq:IIbound}
        \E_{\mathbf{x}_i,\epsilon_i} \left(\widehat{\mathcal{K}}_h(\mathbf{x})-\overline{\mathcal{K}}_h(\mathbf{x})\right)^2 = O\left(n^{-2\alpha/(2\alpha+d_X)}\right).
    \end{equation}
    
    For term $\mathrm{III}$, by \cref{lem:kernelvariance},
    \begin{equation}
        \left|\overline{\mathcal{K}}_h(\mathbf{x})-\mathcal{K}_h(\mathbf{x})\right| = O\left(\left(\log\left(h^{-1}\right)\right)^{3d_X/4}\cdot\sqrt{\frac{\log(4/\delta)}{nh^{d_X}}}\right)
    \end{equation}
    with probability $1-\delta$, for any $\delta>0$. By definition of the kernel, we have the uniform bound $2S$ (recall $S=S_f+S_\mathrm{noise}$), which we may use on the failure probability $\delta$ to get an expectation bound:
    \begin{equation}
        \E_{\mathbf{x}_i,\mathbf{x}\sim\mathcal{U}(X)}\left(\overline{\mathcal{K}}_h(\mathbf{x})-\mathcal{K}_h(\mathbf{x})\right)^2 \leq (1-\delta)\cdot O\left(\left(\log\left(h^{-1}\right)\right)^{3d_X/2}\cdot\frac{\log(4/\delta)}{nh^{d_X}}\right) + \delta(2S)^2.
    \end{equation}
    
    Recall that $\alpha\leq 1$. We explicitly set $\delta=4h^2$, whereupon the bound simplifies to
    \begin{align}
        \E_{\mathbf{x}_i,\mathbf{x}}\left(\overline{\mathcal{K}}_h(\mathbf{x})-\mathcal{K}_h(\mathbf{x})\right)^2 &= O\left(\frac{(\log(h^{-1}))^{1+\frac{3d_X}{2}}}{nh^{d_X}}\right) \notag \\ 
        &= O\left(n^{-\frac{2\alpha}{2\alpha+d_X}}(\log n)^{1+\frac{3d_X}{2}}\right) \label{eq:IIIbound}
    \end{align}
    The first equality uses that $0<h<1/2$; the second line follows from actually substituting $h=n^{-\frac{1}{2\alpha+d_X}}$. 

    Finally, \cref{lem:kernelbias} directly gives us a bound on term $\mathrm{IV}$, which inherently does not involve support points:
    \begin{equation}
        \E_{\mathbf{x}\sim \mathcal{U}(X)}\left(\mathcal{K}_h(\mathbf{x})-f(\mathbf{x})\right)^2 = O(h^{2\alpha}(\log(h^{-1}))^2).
    \end{equation}
    
    Substituting in our choice of $h$ yields
    \begin{equation}\label{eq:IVbound}
        \E_{\mathbf{x}\sim \mathcal{U}(X)}\left(\mathcal{K}_h(\mathbf{x})-f(\mathbf{x})\right)^2 = O\left(n^{-\frac{2\alpha}{2\alpha+d_X}}(\log n)^2\right).
    \end{equation}

    Assembling \eqref{eq:Ibound},\eqref{eq:IIbound},\eqref{eq:IIIbound}, and\eqref{eq:IVbound} together, we arrive at the total bound:
    \begin{align*}
        &\E_{\mathbf{x}_i,\epsilon_i,\mathbf{x}} (\widetilde{G}(H_{P^*,\mathbf{x}})-f(\mathbf{x}))^2 \\\qquad &= O(n^{-2}) + O\left(n^{-\frac{2\alpha}{2\alpha+d_X}}\right)+O\left(n^{-\frac{2\alpha}{2\alpha+d_X}}(\log n)^{1+\frac{3d_X}{2}}\right) + O\left(n^{-\frac{2\alpha}{2\alpha+d_X}}(\log n)^2\right).
    \end{align*}
    
    The first term decays the fastest. For the other three terms, since $d_X\geq 1$, the exponents compare as $1+3d_X/2>2>0$ and the middle term slightly dominates:
    \begin{equation}
        \E_{\mathbf{x}_i,\epsilon_i,\mathbf{x}} (\widetilde{G}(H_{P^*,\mathbf{x}})-f(\mathbf{x}))^2 = O\left(n^{-\frac{2\alpha}{2\alpha+d_X}}(\log n)^{1+\frac{3d_X}{2}}\right).
    \end{equation}

    The norm bound for the soft prompt tokens follows directly by substituting our choice of $h$ and $\beta$ into the bound from \cref{thm:promptingkernel} and hiding constants. 
\end{proof}

\section{Random Transformers}\label{sec:random}
In this section, we show that random transformers are universal approximators when appropriately prompted. We first consider random attention parameters, then extend to fully random transformers with random embedding and decoding. 

While our previous deterministic results were for any hidden dimension satisfying $\dh\geq\din+2$, in the random setting we consider the large-width limit where $\dh\gg\din+2$, which is common in practice \citep{Brown20}. 

We recall here a useful result in random matrix theory regarding the singular values of random matrices with Gaussian entries; for standard references, see \citet{Davidson01} or \citet{Vershynin26}.

\begin{lemma}[Nonasymptotic tail bounds on singular values of random Gaussian matrix]\label{lem:gaussiansingularvalues}
Let $A\in \R^{M\times N}$ with $M\geq N$ be a random matrix with iid $\mathcal{N}(0,1)$ entries. Then for every $t\geq 0$,
\begin{equation}
    \mathbb{P}\left(\sigma_{\min}(A) \geq \sqrt{M}-\sqrt{N}-t\right) \geq 1-e^{-t^2/2}
\end{equation}
and
\begin{equation}
    \mathbb{P}\left(\sigma_{\max}(A) \leq \sqrt{M}+\sqrt{N}+t\right) \geq 1-e^{-t^2/2}.
\end{equation}
\end{lemma}

We will usually apply this result to wide matrices, noting that a matrix and its transpose have the same singular values. Some of our applications will involve iid $\mathcal{N}(0,\sigma^2)$ entries; in those cases both bounds scale by $\sigma$.

\subsection{Random Attention Parameters}\label{sec:randomattn}
Instead of directly assuming particular parameters which satisfy \cref{ass:parameters}, we consider attention parameters with Gaussian entries, and show that these satisfy the condition almost surely. This lets us extend the universal approximation theorem for prompting to these random transformers. 

\begin{assumption}[Random attention parameters]\label{ass:randomparams}
    The entries of the attention parameters $W_K$, $W_Q$, and $W_V$ are drawn iid from Gaussian distributions $\mathcal{N}(0,\sigma_{W_K}^2)$, $\mathcal{N}(0,\sigma_{W_Q}^2)$, and $\mathcal{N}(0,\sigma_{W_V}^2)$ respectively.
\end{assumption}

\begin{remark}\label{rem:parameterinit}
    Outside of a few baselines, our numerical experiments in \cref{sec:numerical} will initialize the entries of all random parameters (including those in \cref{ass:randomembdec}) according to $\mathcal{N}(0,1)$, and our theoretical derivations reflect this choice. Since our networks are not trained, we do not use width-dependent initialization schemes which are necessary for gradient-based learning, such as Xavier \citep{Glorot10} or Kaiming \citep{He15}. Note that such initialization variances, typically on the order of $\Theta(1/\dh)$, will change the asymptotic prompt norm bounds throughout \cref{sec:random}. For completeness, we faithfully track the variance dependencies throughout the intermediate derivations, up until the final asymptotic bounds. 
\end{remark}

For the analogous universal approximation theorem, we need only to prove a random version of \cref{thm:promptingkernel}, since the remaining approximation theory results on kernel regression are independent of transformers. Note that in the random setting, the quantities $\|W\|$ and $\|\mathbf{w}_V\|$, which were treated as constants in \cref{thm:promptingkernel}, are in fact random variables. We adjust our theorem accordingly to make this explicit. 

\begin{theorem}[Prompting NW kernel estimator from random attention transformer]\label{thm:promptingkernelrandomattn}
    Let $X$ satisfy \cref{ass:domain}, $f:X\to\R$ satisfy \cref{ass:function}, and labels satisfy \cref{ass:noise}. Let $G\in\mathcal{G}$ be a transformer network whose attention parameters satisfy \cref{ass:randomparams}. With probability $1$ over the initialization of attention parameters, for any prompt length $n\geq1$ and support points $\{\mathbf{x}_i\in X\}_{i=1}^n$, bandwidth $h>0$, and logit shift $\beta\geq0$, there exists a soft prompt $P^*$ of length $n$ such that 
    \begin{equation}
        \sup_{\mathbf{x}\in X}\left( \widetilde{G}(H_{P^*,\mathbf{x}}) - \widehat{\mathcal{K}}_h(\mathbf{x}) \right)^2 \leq \frac{C_G^2}{n^2e^{2\beta}}\cdot\exp\left(\frac{3R^2}{h^2}\right),
    \end{equation}
    where $C_G = \exp\left((R^2+1)\|W\|\right)\left(\sqrt{R^2+1}\|\mathbf{w}_V\| + S\right)$ is an almost surely finite random constant depending on the realized parameters. Moreover, for any $\delta>0$, the norm of the constructed soft prompt tokens $\|\mathbf{p}^*_i\|$ is bounded as
    \begin{equation}
        \|\mathbf{p}^*_i\| = O\left(\frac{\beta+h^{-2}}{\sqrt{\dh}}\right)
    \end{equation}
    with probability at least $1-\delta$ over the initialization of the attention parameters and uniformly over the realizations of the samples, where the hidden constant can depend on $\delta$. 
\end{theorem}

\begin{proof}
    We prove existence first, then derive the corresponding norm bound. 

    \paragraph{Existence.} By \cref{thm:promptingkernel}, it suffices to show that \cref{ass:randomparams} implies $\widetilde{W}$ satisfies the full row rank condition of \cref{ass:parameters} with probability $1$. 

    Recall $\widetilde{W}\in\R^{(\din+2)\times\dh}$ is defined by vertically concatenating the first $\din+1$ rows of the query-key matrix $W=W_QW_K^T/\sqrt{\dh}$ with the last column of the value matrix $W_V$ transposed:
    \begin{equation}
        \widetilde{W}=\begin{bmatrix} W_{1:\din} \\ W_{\din+1} \\ \mathbf{w}_V^T \end{bmatrix}.
    \end{equation}

    We first show that the key matrix $W_K$ is invertible almost surely. This is a standard result for Gaussian random matrices (which $W_K$ is, by \cref{ass:randomparams}), but we include a proof for completeness. The determinant of $W_K$ is a polynomial in $\dh^2$ variables, and is clearly not identically zero. Its vanishing locus forms a lower-dimensional algebraic variety, thus has Lebesgue measure zero in $\R^{\dh^2}$. Integrating the joint PDF over the measure zero set yields $\mathbb{P}(\det(W_K)=0)=0$, and $W_K$ is invertible almost surely. 

    Next, consider the wide matrix $W_{Q,{1:\din+1}}\in\R^{(\din+1)\times\dh}$, the first $\din+1$ rows of $W_Q$. By \cref{ass:randomparams}, following the same argument as above, any $((\din+1)\times(\din+1))$-submatrix has nonzero determinant almost surely. Hence $\operatorname{rank}(W_{Q,{1:\din+1}})=\din+1$ almost surely.
    
    The transpose of an invertible matrix is invertible, and right-multiplication by an invertible matrix (or a nonzero scalar) preserves row rank. Therefore, we may conclude that the first $\din+1$ rows of $W$ have rank $\din+1$ almost surely. Let $\mathcal{V}$ denote that row space. 
 
    Finally, $W_V$ is initialized independently of $W_Q$ and $W_K$, thus its last column $\mathbf{w}_V$ is a random vector independent of $\mathcal{V}$. Noting the dimensions, $\mathcal{V}$ must be a proper subspace of $\R^{\dh}$, thus has Lebesgue measure zero in $\R^{\dh}$. Integrating the respective PDF over a set of measure zero, we find $\mathbb{P}(\mathbf{w}_V\in\mathcal{V}\mid\mathcal{V})=0$, thus concatenating $\mathbf{w}_V$ will increase rank almost surely. 

    Put together, we conclude that $\operatorname{rank}(\widetilde{W})=\din+2$ (the condition of \cref{ass:parameters}) with probability 1, after which we may directly invoke the construction with the Moore-Penrose pseudoinverse from \cref{thm:promptingkernel}, which yields the desired soft prompt and corresponding bounds:
    \begin{equation}
        \sup_{\mathbf{x}\in X}\left|\widetilde{G}(H_{P^*,\mathbf{x}}) - \widehat{\mathcal{K}}_h(\mathbf{x})\right| \leq \frac{1}{ne^\beta}\exp\left(\frac{3R^2}{2h^2}\right)\exp\left((R^2+1)\|W\|\right)\left(\sqrt{R^2+1}\|\mathbf{w}_V\| + S\right).
    \end{equation}

    A small caveat, as noted above: in the random case, the last two terms which involve $\|W\|$ and $\|\mathbf{w}_V\|$ are no longer constants, but random variables (which grow like $\sqrt{\dh}$, by \cref{lem:gaussiansingularvalues}). Accordingly, we define a new explicit random constant
    \begin{equation}\label{eq:randomconstant}
        C_G \coloneqq \exp\left((R^2+1)\|W\|\right)\left(\sqrt{R^2+1}\|\mathbf{w}_V\| + S\right), 
    \end{equation}
    which is almost surely finite. With this, we rewrite our error as 
    \begin{equation}
        \sup_{\mathbf{x}\in X}\left|\widetilde{G}(H_{P^*,\mathbf{x}}) - \widehat{\mathcal{K}}_h(\mathbf{x})\right| \leq \frac{C_G}{ne^\beta}\exp\left(\frac{3R^2}{2h^2}\right).
    \end{equation}
    Squaring both sides gives the squared error bound in the theorem. 

    \paragraph{Norm bound.}
    Recall from the proof of \cref{thm:promptingkernel}, we have that 
    \begin{equation}
        \|\mathbf{p}_i^*\| \leq \|\widetilde{W}^\dagger\|\|\mathbf{v}_i\| \leq \frac{1}{\sigma_{\min}(\widetilde{W})}\sqrt{\frac{R^2}{h^4} + \left(\beta+\frac{R^2}{2h^2}\right)^2 + S^2}.
    \end{equation}
    The only component that changes in the random attention case is the smallest singular value $\sigma_{\min}(\widetilde{W})$, which is now a random variable we must bound in probability. Fix $\delta>0$. 

    We first consider $W_{1:\din+1}$, the first $\din+1$ rows of $\widetilde{W}$. By definition of $W$, it can be written as
    \begin{equation}
        W_{1:\din+1} = \frac{1}{\sqrt{\dh}}W_{Q,1:\din+1}W_K^T.
    \end{equation}

    Note that $W_{Q,1:\din+1}\in\R^{(\din+1)\times\dh}$ is a wide Gaussian matrix. Invoking \cref{lem:gaussiansingularvalues}, for any $t>0$, 
    \begin{equation}
        \mathbb{P}(\sigma_{\min}(W_{Q,1:\din+1}) \geq \sigma_{W_Q}(\sqrt{\dh}-\sqrt{\din+1}-t)) \geq 1-e^{-t^2/2}.
    \end{equation}
    Picking $t=\sqrt{2\log(2/\delta)}$ and with $\dh\gg\din$ (at least $\dh\geq8(\din+1+2\log(2/\delta)))$, we have
    \begin{equation}
        \sigma_{\min}(W_{Q,1:\din+1}) = \Omega\left(\sigma_{W_Q}\sqrt{\dh}\right)
    \end{equation}
    with probability at least $1-\delta/2$. 

    Conditioned on $W_{Q,1:\din+1}$, the product $W_{Q,1:\din+1}W_K^T/\sqrt{\dh}$ is a random matrix with independent Gaussian columns. The final row of $\widetilde{W}$ is $\mathbf{w}_V^T$, which is completely independent of $W_Q$ and $W_K$. Together, $\widetilde{W}$ has the block-diagonal covariance structure
    \begin{equation}
        \Sigma \coloneqq \begin{bmatrix}
            \frac{\sigma_{W_K}^2}{\dh}W_{Q,1:\din+1}W_{Q,1:\din+1}^T & 0 \\ 0 & \sigma_{W_V}^2 
        \end{bmatrix}.
    \end{equation}

    Conditioned on $W_Q$, we may write (reparameterization trick)
    \begin{equation}
        \widetilde{W} \stackrel{d}{=} \Sigma^{1/2}Z,
    \end{equation}
    where $Z\in\R^{(\din+2)\times\dh}$ is a random matrix with iid standard normal entries. 

    Since $\Sigma^{1/2}$ is square and invertible and $Z$ is wide with full row rank almost surely, supermultiplicativity of the minimum singular value implies
    \begin{equation}
        \sigma_{\min}(\widetilde{W}) \geq \sigma_{\min}(\Sigma^{1/2})\sigma_{\min}(Z).
    \end{equation}

    For block-diagonal matrices, their minimum singular value is simply the minimum of the singular values for each respective block:
    \begin{equation}
        \sigma_{\min}(\Sigma^{1/2}) = \min\left(\frac{\sigma_{W_K}}{\sqrt{\dh}}\sigma_{\min}(W_{Q,1:\din+1}),\sigma_{W_V}\right).
    \end{equation}

    On the aforementioned high probability event, this reduces to
    \begin{equation}
        \sigma_{\min}(\Sigma^{1/2}) \geq \min(c_1\sigma_{W_Q}\sigma_{W_K},\sigma_{W_V})
    \end{equation}
    
    Note $Z\in\R^{(\din+2)\times\dh}$ is a wide matrix since $\dh\gg\din+2$. As above, by \cref{lem:gaussiansingularvalues} we deduce that
    \begin{equation}
        \sigma_{\min}(Z) \geq \Omega\left(\sqrt{\dh}\right)
    \end{equation}
    with probability at least $1-\delta/2$.

    Union bounding with the first high probability event, we have that with probability at least $1-\delta$,
    \begin{equation}
        \sigma_{\min}(\widetilde W)\geq \sigma_{\min}(\Sigma^{1/2})\,\sigma_{\min}(Z) = \Omega\left(\min(\sigma_{W_Q}\sigma_{W_K},\sigma_{W_V})\sqrt{\dh}\right).
    \end{equation}

    Putting everything together, we conclude
    \begin{equation}
    \|\mathbf{p}_i^*\| \;\leq\; \frac{\|\mathbf v_i\|}{\sigma_{\min}(\widetilde W)}
    \;\leq\; \frac{\sqrt{\dfrac{R^2}{h^4} + \left(\beta+\dfrac{R^2}{2h^2}\right)^2 + S^2}}{\sigma_{\min}(\widetilde W)}
    \;=\; O\!\left(\frac{\beta+h^{-2}}{\sqrt{\dh}}\right)
    \end{equation}
    with probability at least $1-\delta$.
\end{proof}

As a corollary, we get a universal approximation theorem in the random setting analogous to \cref{thm:univapprox}.

\begin{theorem}[Universal approximation for prompting random attention transformer]\label{thm:univapproxrandomattn}
    Let $G\in\mathcal{G}$ be a transformer network whose attention parameters satisfy \cref{ass:randomparams}. Suppose $X$ satisfies \cref{ass:domain} and fix an $f$ satisfying \cref{ass:function}. Suppose the support points and query satisfy \cref{ass:supportpoints}, and labels satisfy \cref{ass:noise}. For almost every initialization of the attention parameters, there exists a map, depending on the realized parameters, assigning to each realization of samples $\{\mathbf{x}_i,y_i=f(\mathbf{x}_i)+\epsilon_i\}_{i=1}^n$ a soft prompt $P^*$ of length $n$, independent of the query $\mathbf{x}$, such that 
    \begin{equation}
        \E_{\mathbf{x}_i,\epsilon_i,\mathbf{x}} (\widetilde{G}(H_{P^*,\mathbf{x}})-f(\mathbf{x}))^2 = O\left(n^{-\frac{2\alpha}{2\alpha+d_X}}(\log n)^{1+\frac{3d_X}{2}}\right),
    \end{equation} 
    conditional on the realized parameters. Furthermore, for any $\delta>0$, with probability at least $1-\delta$ over the attention parameters, the soft prompt tokens satisfy 
    \begin{equation}
        \|\mathbf{p}_i^*\| = O\left(\frac{n^{\frac{2}{2\alpha+d_X}}}{\sqrt{\dh}}+1\right)
    \end{equation}
    for all $1\leq i\leq n$, uniformly over the realizations of the samples.
\end{theorem}

\begin{proof}
    We have the same four term decomposition from \cref{thm:univapprox}. Terms $\mathrm{II}$, $\mathrm{III}$, and $\mathrm{IV}$ are inherent to kernel regression, holding independently of the transformer. We use \cref{thm:promptingkernelrandomattn} to bound term $\mathrm{I}$ in the random case. Recall \eqref{eq:expandedapproxerror} from the proof of \cref{thm:promptingkernel}:
    \begin{equation}
        \sup_{\mathbf{x}\in X}\left|\widetilde{G}(H_{P^*,\mathbf{x}}) - \widehat{\mathcal{K}}_h(\mathbf{x})\right| \leq \frac{1}{ne^\beta}\exp\left(\frac{3R^2}{2h^2}\right)\exp\left((R^2+1)\|W\|\right)\left(\sqrt{R^2+1}\|\mathbf{w}_V\| + S\right).
    \end{equation}

    In the proof of \cref{thm:univapprox}, we chose $h=n^{-\frac{1}{2\alpha+d_X}}$ and correspondingly set $\beta=\frac{3R^2}{2}\cdot n^{\frac{2}{2\alpha+d_X}}$; the remaining constant was absorbed into the $O(\cdot)$. In the random attention case this constant is a random variable, which we defined as $C_G$ in \eqref{eq:randomconstant}, and will in fact depend exponentially on $\sqrt{\dh}$. To deal with this, we may redefine $\beta$ to exactly absorb it:
    \begin{equation}
        \beta = \frac{3R^2}{2}\cdot n^{\frac{2}{2\alpha+d_X}} + \log_+ C_G.
    \end{equation}
    We restrict to $\log_+ x \coloneqq \max(\log x,0)$, the positive part of the logarithm, to avoid edge cases which cause negative $\beta$. Substituting in this new $\beta$ gives the desired $\leq n^{-2}$ bound for term $\mathrm{I}$ almost surely. The remaining steps follow the proof of \cref{thm:univapprox}, mutatis mutandis.

    Fix $\delta>0$. The token norm bound follows by inserting $h$ and $\beta$ into the token bound from \cref{thm:promptingkernelrandomattn}, to which we allocate failure probability $\delta/2$. By \cref{lem:gaussiansingularvalues}, with high probability we have that 
    \begin{equation}
        \|W\|\leq\frac{1}{\sqrt{\dh}}\|W_Q\|\|W_K\|=O(\sigma_{W_Q}\sigma_{W_K}\sqrt{\dh})
    \end{equation} 
    and $\|\mathbf{w}_V\|=O(\sigma_{W_V}\sqrt{\dh})$, whence we find
    \begin{equation}
        \log_+\left(\sqrt{R^2+1}\|\mathbf{w}_V\| + S\right)=O(\log(\sigma_{W_V}\sqrt{\dh})).
    \end{equation}
    We distribute the other $\delta/2$ failure probability over these events. Consequently, the change in $\beta$ only results in an added $O(1)$ term to the norm bound (hiding the variance dependence):
    \begin{equation}
        \frac{\log_+ C_G}{\sqrt{\dh}} \leq \frac{(R^2+1)\|W\| +\log_+\left(\sqrt{R^2+1}\|\mathbf{w}_V\| + S\right)}{\sqrt{\dh}} = O(1)
    \end{equation}
    with probability at least $1-\delta/2$. The union bound gives the claimed token norm bound with probability at least $1-\delta$.
\end{proof}

\begin{remark}
    Unlike before in \cref{thm:univapprox}, the logit shift is now adaptive to the realized network (in particular, the realization of the random attention parameters), with only an additional $O(1)$ cost in token norm.
\end{remark}

\subsection{Random Embedding and Decoding}\label{sec:randomembdec}

For fully random transformers, we expand our transformer class. The fixed embedding and decoding defined in \cref{sec:architecture} were used primarily for clarity and concreteness. Here, we redefine them as affine maps, which will also be initialized randomly alongside the attention parameters. 

The embedding is now given by
\begin{equation}\label{eq:affemb}
    E_\mathrm{aff}(\mathbf{x}) = \xemb \coloneqq \mathbf{x}^TU_E+\mathbf{b}_E^T,
\end{equation}
where $U_E\in\R^{\din\times\dh}$ and $\mathbf{b}_E\in\R^{\dh}$. Similarly, letting $H_{n+1}$ denote the last row of a matrix $H$, the decoding is redefined as
\begin{equation}\label{eq:affdec}
    D_{\mathrm{aff}}(H)=H_{n+1}\mathbf{u}_D+b_D,
\end{equation}
where $\mathbf{u}_D\in\R^{\dh}$ and $b_D\in\R$. The bias is necessary in the embedding, but optional for decoding, as we will see in the proof of \cref{thm:promptingkernelrandomtransformer}. Abusing notation, we define $\widetilde{G}=D_{\mathrm{aff}}\circ A$ like before. 

We still consider a transformer model with a single attention layer. Our new transformer class is given by
\begin{equation}
    \mathcal{G}_\mathrm{aff} = \left\{G=D_{\mathrm{aff}}\circ A\circ E_{\mathrm{aff}} \mid W_Q,W_K,W_V\in\R^{\dh\times\dh},  U_E\in\R^{\din\times\dh}, \mathbf{b}_E,\mathbf{u}_D\in\R^{\dh}, b_D\in\R\right\}.
\end{equation}

We refer to the totality of the weights as the \emph{transformer weights}, to distinguish them from just the attention parameters as before. We have a similar randomness assumption for the new embedding and decoding parameters.

\begin{assumption}[Random embedding and decoding parameters]\label{ass:randomembdec}
    The embedding parameters $U_E$ and $\mathbf{b}_E$ have entries distributed iid according to Gaussian distributions $\mathcal{N}(0,\sigma_{U_E}^2)$ and $\mathcal{N}(0,\sigma_{\mathbf{b}_E}^2)$ respectively. The decoding parameter $\mathbf{u}_D$ has entries distributed iid according to $\mathcal{N}(0,\sigma_{\mathbf{u}_D}^2)$ and $b_D\sim\mathcal{N}(0,\sigma_{b_D}^2)$.
\end{assumption}

Even in the fully random setting, prompting can still reconstruct the NW kernel estimator, as the following theorem shows. 

\begin{theorem}[Prompting NW kernel estimator for fully random transformer]\label{thm:promptingkernelrandomtransformer}
    Let $X$ satisfy \cref{ass:domain} and $f:X\to\R$ satisfy \cref{ass:function}. Assume labels satisfy \cref{ass:noise}. Fix $G\in \mathcal{G}_\mathrm{aff}$ whose attention parameters satisfy \cref{ass:randomparams} and embedding and decoding parameters satisfy \cref{ass:randomembdec}. With probability $1$ over the initialization of the transformer weights, for any prompt length $n\geq1$, set of support points $\{\mathbf{x}_i\in X\}_{i=1}^n$, bandwidth $h>0$, and logit shift $\beta\geq0$, there exists a soft prompt $P^*$ of length $n$ such that 
    \begin{equation}
        \sup_{\mathbf{x}\in X}\left( \widetilde{G}(H_{P^*,\mathbf{x}}) - \widehat{\mathcal{K}}_h(\mathbf{x}) \right)^2 \leq \frac{C_{G,\mathrm{aff}}^2}{n^2e^{2\beta}}\cdot\exp\left(\frac{3R^2}{h^2}\right),
    \end{equation}
    where 
    \begin{equation}
        C_{G,\mathrm{aff}} \coloneqq \exp\left(\left(R\|U_E\|+\|\mathbf{b}_E\|\right)^2\|W\|\right)\left(\left(R\|U_E\|+\|\mathbf{b}_E\|\right)\|W_V\|\|\mathbf{u}_D\|+|b_D|+S\right)
    \end{equation}
    is an almost surely finite random constant depending on the realized transformer weights. Moreover, for any $\delta>0$, the constructed soft prompt has tokens $\|\mathbf{p}^*_i\|$ satisfying
    \begin{equation}
        \|\mathbf{p}^*_i\| = O\left(\frac{\beta+h^{-2}}{\dh}\right),
    \end{equation}
    with probability at least $1-\delta$, uniformly over the realizations of the samples. 
\end{theorem}

\begin{proof}
    We closely follow the proof structure of \cref{thm:promptingkernel}, adapting it to this random setting. Recall that the attention logit for each prompt token $\mathbf{p}_i$ looks like $\xemb W\mathbf{p}_i$, where $W=\frac{1}{\sqrt{\dh}}W_QW_K^T$. Substituting our new affine embedding defined in \eqref{eq:affemb}, the logit expands as
    \begin{equation}
        \xemb W\mathbf{p}_i = (\mathbf{x}^TU_E+\mathbf{b}_E^T)W\mathbf{p}_i = \mathbf{x}^T(U_EW)\mathbf{p}_i+\mathbf{b}_E^TW\mathbf{p}_i.
    \end{equation}

    Further recall that to mimic the NW kernel estimator \eqref{eq:kernelformula}, we want this logit to match $\frac{\mathbf{x}^T\mathbf{x}_i}{h^2}-\frac{\|\mathbf{x}_i\|^2}{2h^2}+\beta$, which yields the following two conditions:
    \begin{equation}
        U_EW\mathbf{p}_i = \frac{\mathbf{x}_i}{h^2} \qquad \text{ and } \qquad \mathbf{b}_E^TW\mathbf{p}_i = -\frac{\|\mathbf{x}_i\|^2}{2h^2}+\beta.
    \end{equation}

    We also wanted to match the value part, which is affected by the new decoding. The unnormalized contribution of a prompt token $\mathbf{p}_i$ to the output is $\mathbf{p}_i^TW_V\mathbf{u}_D+b_D$, which suggests the following constraint:
    \begin{equation}
        (W_V\mathbf{u}_D)^T\mathbf{p}_i=y_i-b_D.
    \end{equation}

    We consolidate the conditions into a single system
    $\widetilde{W}_\mathrm{aff}\mathbf{p}_i=\mathbf{v}'_i$, where 
    \begin{equation}\label{eq:wandvaff}
        \widetilde{W}_\mathrm{aff} = \begin{bmatrix} U_EW \\ \mathbf{b}_E^TW \\ (W_V\mathbf{u}_D)^T \end{bmatrix}\in\R^{(\din+2)\times\dh} \qquad \text{ and } \qquad \mathbf{v}'_i = \begin{bmatrix} \frac{\mathbf{x}_i}{h^2} \\ -\frac{\|\mathbf{x}_i\|^2}{2h^2} + \beta \\ y_i-b_D \end{bmatrix}\in\R^{\din+2}.
    \end{equation}

    Like in the proof of \cref{thm:promptingkernelrandomattn}, it suffices to show $\widetilde{W}_\mathrm{aff}$ has full row rank almost surely. Since $W_Q$ and $W_K$ are initialized with independent Gaussian entries, $W$ has rank $\dh$ almost surely. Similarly, the distributions of $U_E$ and $\mathbf{b}_E$ imply 
    \begin{equation}
        \begin{bmatrix}
            U_E \\ \mathbf{b}_E^T 
        \end{bmatrix}\in\R^{(\din+1)\times\dh}
    \end{equation}
    has full row rank $\din+1$ almost surely, which is preserved under right multiplication by the full-rank $W$. 
    
    Lastly, consider $W_V\mathbf{u}_D$, a nondegenerate (since $\mathbf{u}_D\neq 0$ almost surely) random vector independent of the previous rows of $\widetilde{W}_\mathrm{aff}$. The previous rows span a proper subspace of $\R^{\dh}$ (Lebesgue measure $0$), thus the probability that $W_V\mathbf{u}_D$ is in the subspace is $0$, meaning its inclusion strictly increases the rank. 

    We conclude the system has at least one solution; we construct $P^*$ using the Moore-Penrose pseudoinverse. It remains to uniformly bound the difference $\left| \widetilde{G}(H_{P^*,\mathbf{x}}) - \widehat{\mathcal{K}}_h(\mathbf{x}) \right|$. As we did in \eqref{eq:outputkerneldifference}, we expand our attention output with the constructed prompt and split it into individual terms:
    \begin{equation}\label{eq:outputkernelsdifferenceaffine}
        \left| \widetilde{G}(H_{P^*,\mathbf{x}}) - \widehat{\mathcal{K}}_h(\mathbf{x}) \right| \leq \frac{1}{e^\beta}\cdot\frac{\left| \exp(\xemb W\xemb^T)\right|\left(|\xemb W_V\mathbf{u}_D+b_D| + \left|\widehat{\mathcal{K}}_h(\mathbf{x})\right|\right)}
        {\sum_{i=1}^n\exp\left(\frac{\mathbf{x}^T \mathbf{x}_i}{h^2}-\frac{\|\mathbf{x}_i\|^2}{2h^2}\right)}.
    \end{equation}

    We retain the bounds on the deterministic parts, i.e. 
    \begin{equation}
        \sum_{i=1}^n\exp\left(\frac{\mathbf{x}^T \mathbf{x}_i}{h^2}-\frac{\|\mathbf{x}_i\|^2}{2h^2}\right) \geq n\exp\left(-\frac{3R^2}{2h^2}\right) \quad \text{ and }\quad \left|\widehat{\mathcal{K}}_h(\mathbf{x})\right|\leq S.
    \end{equation}

    The norm of the embedded query is bounded as 
    \begin{equation}
        \|\xemb\|=\|\mathbf{x}^TU_E+\mathbf{b}_E^T\|\leq R\|U_E\|+\|\mathbf{b}_E\|,
    \end{equation}
    from which it follows that
    \begin{equation}
        |\xemb W\xemb^T| \leq \|\xemb\|^2\|W\|\leq (R\|U_E\|+\|\mathbf{b}_E\|)^2\|W\|.
    \end{equation}

    The decoding part is bounded similarly:
    \begin{equation}
        |\xemb W_V\mathbf{u}_D+b_D| \leq \|\xemb\|\|W_V\|\|\mathbf{u}_D\|+|b_D| \leq (R\|U_E\|+\|\mathbf{b}_E\|)\|W_V\|\|\mathbf{u}_D\|+|b_D|.
    \end{equation}

    As before, we define a new variable to summarize the random constants, which is almost surely finite:
    \begin{equation}\label{eq:randomconstantaff}
        C_{G,\mathrm{aff}} \coloneqq \exp\left(\left(R\|U_E\|+\|\mathbf{b}_E\|\right)^2\|W\|\right)\left(\left(R\|U_E\|+\|\mathbf{b}_E\|\right)\|W_V\|\|\mathbf{u}_D\|+|b_D|+S\right).
    \end{equation}
    Though we do not use it here, note that by \cref{lem:gaussiansingularvalues} we may derive 
    \begin{equation}\label{eq:randomconstantaffbound}
        C_{G,\mathrm{aff}} = \exp\left(O(\dh^{1.5})\right)
    \end{equation} 
    with high probability. As before, we will absorb $C_{G,\mathrm{aff}}$ into the logit shift later in the corresponding universal approximation theorem.

    Plugging everything into \eqref{eq:outputkernelsdifferenceaffine} and squaring both sides, the final bound simplifies compactly as 
    \begin{equation}
        \sup_{\mathbf{x}\in X}\left( \widetilde{G}(H_{P^*,\mathbf{x}}) - \widehat{\mathcal{K}}_h(\mathbf{x}) \right)^2 \leq \frac{C_{G,\mathrm{aff}}^2}{n^2e^{2\beta}}\cdot\exp\left(\frac{3R^2}{h^2}\right).
    \end{equation}

    \paragraph{Norm bound.} 
    As in the proof of \cref{thm:promptingkernel}, the minimum-norm solution satisfies
    \begin{equation}
        \|\mathbf{p}_i^*\| \leq \|\widetilde{W}_\mathrm{aff}^\dagger\|\|\mathbf{v}'_i\| 
        = \frac{\|\mathbf{v}'_i\|}{\sigma_{\min}(\widetilde{W}_\mathrm{aff})}.
    \end{equation}
    
    Compared to \eqref{eq:wandv}, the only coordinate in $\mathbf{v}'_i$ that differs is the last, which is now $y_i-b_D$ (versus just $y_i$). A standard Gaussian tail bound yields $|b_D|\leq\sigma_{b_D}\sqrt{2\log(10/\delta)}$ with probability at least $1-\delta/5$. This does not change the asymptotic bound: on this event,
    \begin{equation}
        \|\mathbf{v}'_i\| \leq \sqrt{\frac{R^2}{h^4} + \left(\beta+\frac{R^2}{2h^2}\right)^2 
        + \left(S+\sigma_{b_D}\sqrt{2\log(10/\delta)}\right)^2} = O(\beta+h^{-2})
    \end{equation}
    for all $1\leq i\leq n$, where the hidden constant may depend on $\delta$. 

    For the smallest singular value, we first write
    \begin{equation}
        M \coloneqq \begin{bmatrix} U_E \\ \mathbf{b}_E^T \end{bmatrix}
        = \Lambda Z_1 \in\R^{(\din+1)\times\dh}
    \end{equation}
    where $\Lambda \coloneqq \operatorname{diag}(\sigma_{U_E},\ldots,\sigma_{U_E},\sigma_{\mathbf{b}_E})$ and $Z_1\in\R^{(\din+1)\times\dh}$ has iid standard normal entries. The first $\din+1$ rows of $\widetilde{W}_\mathrm{aff}$ can then be written as 
    \begin{equation}
        MW=\frac{1}{\sqrt{\dh}}MW_QW_K^T,
    \end{equation}
    and the $j$-th column of $MW$ can be written 
    \begin{equation}
        \frac{1}{\sqrt{\dh}}MW_QW_{K,j}^T,
    \end{equation}
    where $W_{K,j}$ denotes the $j$-th row of $W_K$. The $W_{K,j}$ are iid random vectors with covariance $\sigma_{W_K}^2 I_{\dh}$. Similarly, the $j$-th entry of the last row of $\widetilde{W}_\mathrm{aff}$ is $W_{V,j}\mathbf{u}_D$ where $W_{V,j}$ is the $j$-th row of $W_V$. These entries are iid from $\mathcal{N}(0,\sigma_{W_V}^2\|\mathbf{u}_D\|^2)$, independent of $W_K$. Together, conditioned on $M$, $W_Q$, and $\mathbf{u}_D$, the columns of $\widetilde{W}_\mathrm{aff}$ are iid Gaussian with block-diagonal covariance
    \begin{equation}
        \Sigma = \begin{bmatrix}
            \frac{\sigma_{W_K}^2}{\dh}(MW_Q)(MW_Q)^T & 0 \\ 
            0 & \sigma_{W_V}^2\|\mathbf{u}_D\|^2 
        \end{bmatrix}.
    \end{equation}
    Again, we write $\widetilde{W}_\mathrm{aff} \stackrel{d}{=} \Sigma^{1/2}Z$ for $Z\in\R^{(\din+2)\times\dh}$ with iid standard normal entries. By supermultiplicativity of the minimum singular value in this case and the block-diagonal structure,
    \begin{equation}\label{eq:sigmamindecomp}
        \sigma_{\min}(\widetilde{W}_\mathrm{aff}) 
        \geq \min\left(\frac{\sigma_{W_K}}{\sqrt{\dh}}\sigma_{\min}(MW_Q),\; 
        \sigma_{W_V}\|\mathbf{u}_D\|\right)\cdot\sigma_{\min}(Z).
    \end{equation}

    It remains to lower bound each random quantity. Conditioning on $M$, the columns of $MW_Q$ are iid $\mathcal{N}(0,\sigma_{W_Q}^2MM^T)$, so $MW_Q\stackrel{d}{=}\sigma_{W_Q}(MM^T)^{1/2}Z_2$ with standard Gaussian $Z_2\in\R^{(\din+1)\times\dh}$, giving
    \begin{equation}
        \sigma_{\min}(MW_Q) \geq \sigma_{W_Q}\,\sigma_{\min}(M)\,\sigma_{\min}(Z_2)
        \geq \sigma_{W_Q}\min(\sigma_{U_E},\sigma_{\mathbf{b}_E})\,\sigma_{\min}(Z_1)\,\sigma_{\min}(Z_2).
    \end{equation}

    Each of $Z_1, Z_2$, and $Z$ is a wide standard Gaussian matrix; we apply \cref{lem:gaussiansingularvalues} individually. This yields 
    \begin{equation}
        \sigma_{\min}(Z_1)=\Omega\left(\sqrt{\dh}\right),\quad \sigma_{\min}(Z_2)=\Omega\left(\sqrt{\dh}\right),\quad \sigma_{\min}(Z)=\Omega\left(\sqrt{\dh}\right),
    \end{equation}
    each with probability at least $1-\delta/5$, where the constants may depend on $\delta$. 

    Finally, \cref{lem:gaussiansingularvalues} also gives $\|\mathbf{u}_D\|=\Omega\left(\sigma_{\mathbf{u}_D}\sqrt{\dh}\right)$ with probability at least $1-\delta/5$. Union bounding over all five events and substituting into \eqref{eq:sigmamindecomp}, we find
    \begin{equation}
        \sigma_{\min}(\widetilde{W}_\mathrm{aff}) = \Omega\left(\min\Big(\sigma_{W_Q}\sigma_{W_K}\min(\sigma_{U_E},\sigma_{\mathbf{b}_E}),\sigma_{W_V}\sigma_{\mathbf{u}_D}\Big)\cdot\dh\right)
    \end{equation}
    with probability at least $1-\delta$. 

    Putting the two bounds together, we arrive at the overall bound: with probability at least $1-\delta$,
    \begin{equation}
        \|\mathbf{p}_i^*\| \leq 
        \frac{\sqrt{\frac{R^2}{h^4} + \left(\beta+\frac{R^2}{2h^2}\right)^2 
        + \left(S+\sigma_{b_D}\sqrt{2\log(10/\delta)}\right)^2}}
        {\sigma_{\min}(\widetilde{W}_\mathrm{aff})}
        = O\left(\frac{\beta+h^{-2}}{\dh}\right).
    \end{equation}
\end{proof}

\begin{remark}
    Notice the improved prompt norm scaling in the fully random case: we have $\dh$ in the denominator, as opposed to $\sqrt{\dh}$ in the fixed embedding and decoding case. The random versions distribute the constraints on each token densely across $\dh$, which scales up $\widetilde{W}_\mathrm{aff}$ by a factor of $\sqrt{\dh}$ relative to $\widetilde{W}$, whose zero padding embedding and coordinate readout decoding limit the constraints to a few rows. We will see in \cref{thm:univapproxrandomtransformer} that this comes at a cost: the query self-attention is also inflated, which requires a larger logit shift to offset it, impacting the token norm bound (see \cref{rem:fullyrandomtokensize}).
\end{remark}

Using this result gives the following analogous universal approximation theorem in the fully random case. 

\begin{theorem}[Universal approximation for prompting fully random transformer]\label{thm:univapproxrandomtransformer}
    Let $G\in\mathcal{G}_\mathrm{aff}$ be a transformer network whose attention parameters satisfy \cref{ass:randomparams} and whose embedding and decoding parameters satisfy \cref{ass:randomembdec}. Suppose $X$ satisfies \cref{ass:domain} and fix an $f$ satisfying \cref{ass:function}. Suppose the support points and query satisfy \cref{ass:supportpoints} with labels satisfying \cref{ass:noise}. For almost every initialization of the parameters, there exists a map assigning to each realization of samples $\{\mathbf{x}_i,y_i=f(\mathbf{x}_i)+\epsilon_i\}_{i=1}^n$ a soft prompt $P^*$ of length $n$, independent of the query $\mathbf{x}$, such that 
    \begin{equation}
        \E_{\mathbf{x}_i,\epsilon_i,\mathbf{x}} (\widetilde{G}(H_{P^*,\mathbf{x}})-f(\mathbf{x}))^2 = O\left(n^{-\frac{2\alpha}{2\alpha+d_X}}(\log n)^{1+\frac{3d_X}{2}}\right),
    \end{equation}
    conditional on the realized parameters. Furthermore, for any $\delta>0$, with probability at least $1-\delta$ over the transformer weights, the soft prompt tokens satisfy 
    \begin{equation}
        \|\mathbf{p}_i^*\| = O\left(\frac{n^{\frac{2}{2\alpha+d_X}}}{\dh}+\sqrt{\dh}\right)
    \end{equation}
    for all $1\leq i\leq n$, uniformly over the realizations of the samples.
\end{theorem}

\begin{proof}
    The proof follows the same approach as \cref{thm:univapproxrandomattn}. In the four term decomposition from the proof of \cref{thm:univapprox}, the latter three are bounded independently of the transformer, so still hold here. 

    For term $\mathrm{I}$, we use the above result of \cref{thm:promptingkernelrandomtransformer}:
    \begin{equation}
        \sup_{\mathbf{x}\in X}\left( \widetilde{G}(H_{P^*,\mathbf{x}}) - \widehat{\mathcal{K}}_h(\mathbf{x}) \right)^2 \leq \frac{C_{G,\mathrm{aff}}^2}{n^2e^{2\beta}}\cdot\exp\left(\frac{3R^2}{h^2}\right).
    \end{equation}

    Like we did in \cref{thm:univapproxrandomattn}, we still set $h=n^{-\frac{1}{2\alpha+d_X}}$, but change $\beta$ to absorb the random constant defined in \eqref{eq:randomconstantaff}:
    \begin{equation}
        \beta=\frac{3R^2}{2}\cdot n^{\frac{2}{2\alpha+d_X}} + \log_+C_{G,\mathrm{aff}}.
    \end{equation}
    Substituting these choices, we are left with exactly 
    \begin{equation}
        \sup_{\mathbf{x}\in X}\left( \widetilde{G}(H_{P^*,\mathbf{x}}) - \widehat{\mathcal{K}}_h(\mathbf{x}) \right)^2 \leq \frac{1}{n^2}
    \end{equation}
    almost surely. Assembling the four terms exactly as in \cref{thm:univapprox}, we get our desired bound. 

    We compute our token bound the same way as in \cref{thm:univapproxrandomattn}, but the different $\beta$ changes our conclusion slightly. Fix $\delta>0$; from \cref{thm:promptingkernelrandomtransformer} we have that
    \begin{equation}
        \|\mathbf{p}^*_i\| = O\left(\frac{\beta+h^{-2}}{\dh}\right)
    \end{equation}
    with probability at least $1-\delta/2$. As mentioned in \eqref{eq:randomconstantaffbound}, using standard Gaussian norm bounds we find that with probability at least $1-\delta/2$,  
    \begin{equation}
        C_{G,\mathrm{aff}} = \exp\left(O(\dh^{1.5})\right),
    \end{equation}
    which means 
    \begin{equation}
        \beta = O\left(n^{\frac{2}{2\alpha+d_X}} + \dh^{1.5}\right).
    \end{equation}

    Therefore, with probability at least $1-\delta$ over the initialization of the transformer weights, 
    \begin{equation}
        \|\mathbf{p}_i^*\| = O\left(\frac{n^{\frac{2}{2\alpha+d_X}}}{\dh}+\sqrt{\dh}\right)
    \end{equation}
    uniformly over the realizations of the samples. 
\end{proof}

\begin{remark}\label{rem:fullyrandomtokensize}
    In contrast with \cref{thm:univapproxrandomattn} where adapting the logit shift to the realized attention parameters came at only an additional $O(1)$ cost in token norm, the random embedding and decoding necessitate a larger $O(\sqrt{\dh})$ increase.
\end{remark}

\section{Numerical Experiments}\label{sec:numerical}
In this section, we perform numerical experiments on synthetic datasets to validate our theoretical findings. Our model architectures remain faithful to the settings where our results were derived, using a single softmax attention layer with no feedforward component. Our experiments involve the two architectural variants used in \cref{sec:randomattn,sec:randomembdec} respectively:
\begin{itemize}
    \item \textbf{Architecture A.} One-layer softmax attention with fixed embedding and decoding, exactly as specified in \cref{sec:architecture}. In particular, we use a homogeneous embedding with additional zero padding and a readout decoder. The attention parameters are randomly initialized and not trained.
    \item \textbf{Architecture B.} One-layer softmax attention with variable affine embedding and decoding, as described at the beginning of \cref{sec:randomembdec}. All transformer weights (attention, embedding, decoding) are randomly initialized and not trained. 
\end{itemize}

As noted in \cref{rem:parameterinit}, all random parameters that will not be trained are initialized according to $\mathcal{N}(0,1)$. Only in the baseline comparisons for the regression rate experiment do we use width-dependent initializations. We perform all softmax computations in log-space. All experiments were run over 10 seeds. 

Our domain is the unit sphere $X=S^2\subset\R^3$ with $d_X=2,\din=3$, $R=1$, and $\tau_X=1$. Support points are drawn iid uniformly on $X$. Queries are either also sampled uniformly, or set over a 2,000-point Fibonacci lattice (when testing suprema). The test function is a unit-norm combination of degree-$4$ real spherical harmonics, drawn once by sampling a coefficient vector for the nine basis functions and fixed throughout. We find that the sampled function $f$ has supremum norm $S_f\approx2.6$, computed on a 200,000-point Fibonacci lattice. Spherical harmonics are smooth, in particular $\alpha=1$. Labels are generated with uniform noise $\mathcal{U}([-1,1])$.

\paragraph{Rank and $\sigma_{\min}$ scaling.}
Sweeping over hidden dimension $\dh$, we check that the relevant coefficient matrices, $\widetilde{W}$ and $\widetilde{W}_{\mathrm{aff}}$ (defined in \eqref{eq:wandv} and \eqref{eq:wandvaff} respectively), have minimum singular values which scale as $\Omega(\sqrt\dh)$ and $\Omega(\dh)$ respectively, and are full rank. The results are shown in \cref{fig:singularvaluescaling}.

\begin{figure}[h]
\centering
\includegraphics[width=.9\linewidth]{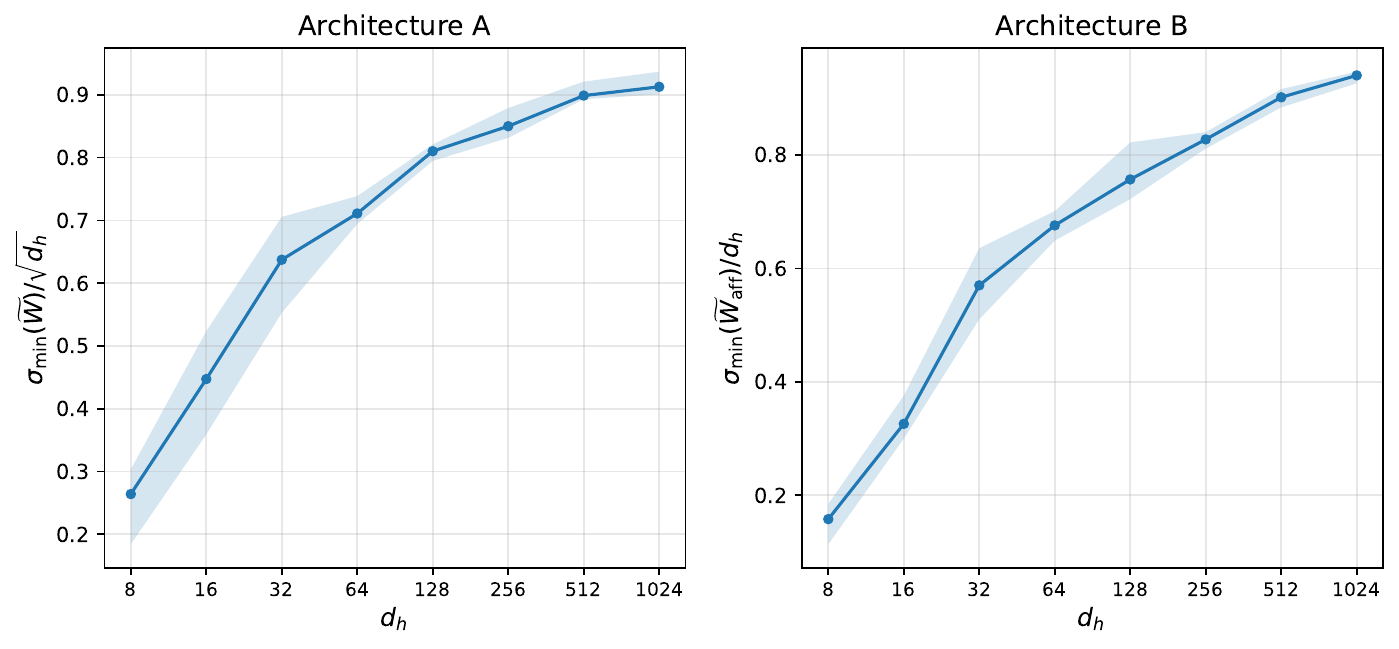}
\caption{Semi-log plot of normalized minimum singular value versus $\dh$ for both architectures. Architecture A normalizes by $\sqrt{\dh}$ and Architecture B normalizes by $\dh$. Solid lines show medians and shaded regions show IQR.}
\label{fig:singularvaluescaling}
\end{figure}

As expected, \cref{fig:singularvaluescaling} exhibits respective normalized minimum singular values concentrating on a positive constant at a shrinking speed. No visible points close to zero supports the almost surely full rank condition.

\paragraph{Kernel approximation.} 
We test the main mechanism of our work: the explicit closed-form construction of a soft prompt to approximate the NW kernel estimator, following \cref{thm:promptingkernelrandomattn}. Across different kernel bandwidths $h$, we observe the scaling in prompt length $n$ and logit shift $\beta$, reporting supremum squared error over a $2000$-point Fibonacci lattice. We test only for Architecture A; the inflated norms of Architecture B cause large fluctuations across seeds which obscure the scaling. The results are shown in \cref{fig:kernelapproxerror}.

\begin{figure*}[h]
    \centering
    \begin{subfigure}[t]{0.45\linewidth}
        \centering
        \includegraphics[width=\linewidth]{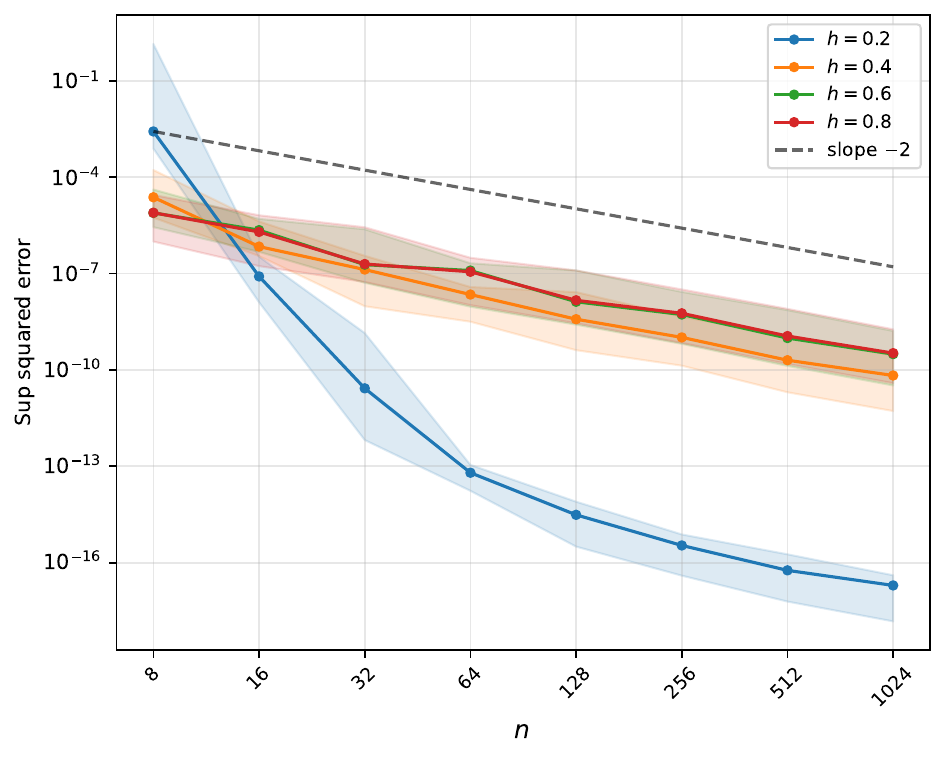}
        \caption{Log-log plot of error versus sample size/prompt length $n$, with fixed $\beta=8$ and $\dh=256$.}
        \label{fig:kernelapproxerrorn}
    \end{subfigure}
    \begin{subfigure}[t]{0.45\linewidth}
        \centering
        \includegraphics[width=\linewidth]{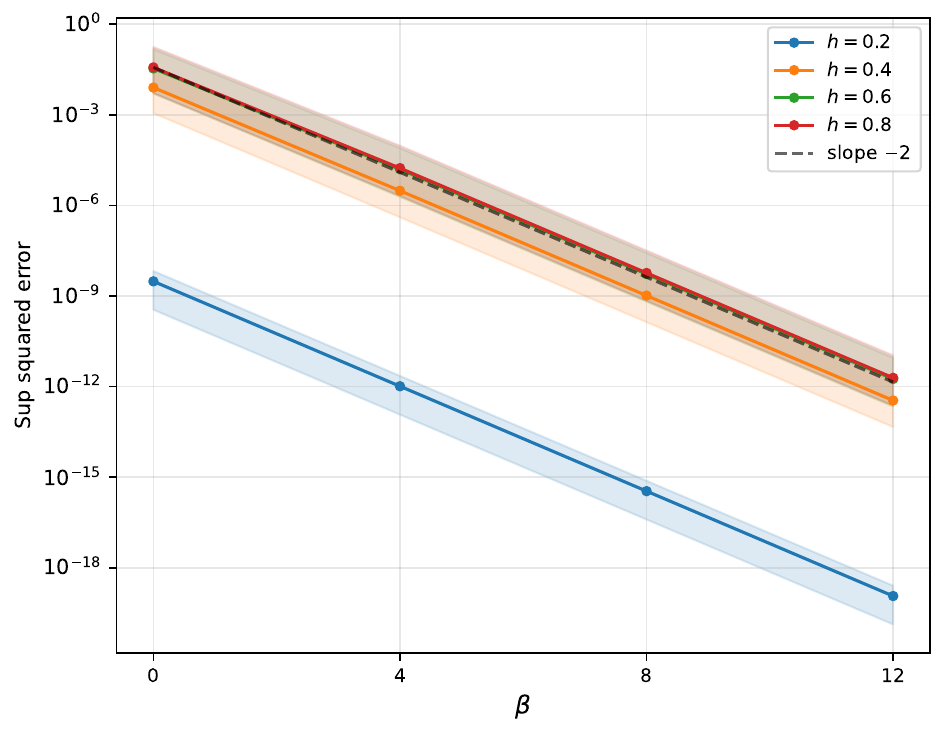}
        \caption{Semi-log plot of error versus logit shift $\beta$, with fixed $n=256$ and $\dh=256$.}
        \label{fig:kernelapproxerrorbeta}
    \end{subfigure}
    \hfill
    \caption{Scaling of supremum squared error between construction and oracle NW estimator across different bandwidths. Solid lines show medians and shaded regions show IQR.}
    \label{fig:kernelapproxerror}
\end{figure*}

\cref{fig:kernelapproxerrorn} supports the predicted scaling in $n$ at the rate $n^{-2}$; we observe an interesting phenomenon for small bandwidth $h=0.2$, where there is an initial steeper drop for small $n$ before aligning with the predicted rate. This is likely attributable to some queries having insufficient local support points in the small $n$ and $h$ regime, so the query self-attention term in the transformer output dominates. In \cref{fig:kernelapproxerrorbeta}, we see the scaling matches almost exactly across all bandwidths (our fixed $n$ is sufficiently large to avoid the aforementioned behavior for $h=0.2$).

\paragraph{Token norm bound.}
We report the largest norm of the constructed soft prompt tokens, with results shown in \cref{fig:tokennormscaling}. Again, we only test Architecture A here. 

\begin{figure}[h]
\centering
\includegraphics[width=.45\linewidth]{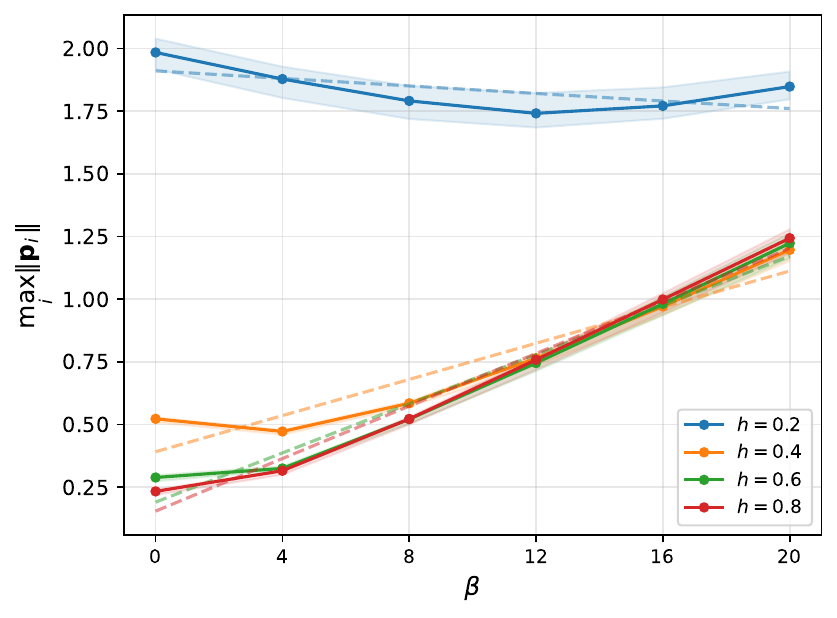}
\caption{Max token norm versus $\beta$ across various bandwidths $h$. Fixed $\dh=256$ and $n=256$. Solid lines show medians and shaded regions show IQR.}
\label{fig:tokennormscaling}
\end{figure}

\newpage

\cref{fig:tokennormscaling} verifies the predicted linear norm growth in logit shift; we again observe nonmonotonic behavior with small bandwidth $h=0.2$, explainable by the cancellation of the particular prompt coordinate $\beta-h^{-2}/2=\beta-12.5$. The other bandwidths demonstrate similar behavior but are not as obvious on the plot (for example, $h=0.4$ dips down at around $4$).

\paragraph{Minimax regression rate.} 
We verify \cref{thm:univapproxrandomattn,thm:univapproxrandomtransformer}, the predicted minimax convergence rate for universal approximation. Exactly as we did in the proofs, we pick $h$ and $\beta$ by explicitly computing the constants $C_G$ and $C_{G,\mathrm{aff}}$, and directly constructing the soft prompts using the pseudoinverses. 

We compare against four benchmarks, the last three of which are architecture dependent (A or B):
\begin{enumerate}
    \item the oracle (empirical) NW estimator;
    \item random frozen weights, trained prompt;
    \item no prompt, trained weights;
    \item frozen random prompt, trained weights.
\end{enumerate}

For all benchmarks, parameters to be trained are initialized with variance $1/\dh$ and trained with Adam; everything else uses the standard variance $1$ initialization. The results are shown in \cref{fig:minimax}. 

\begin{figure}[h]
\centering
\includegraphics[width=0.9\linewidth]{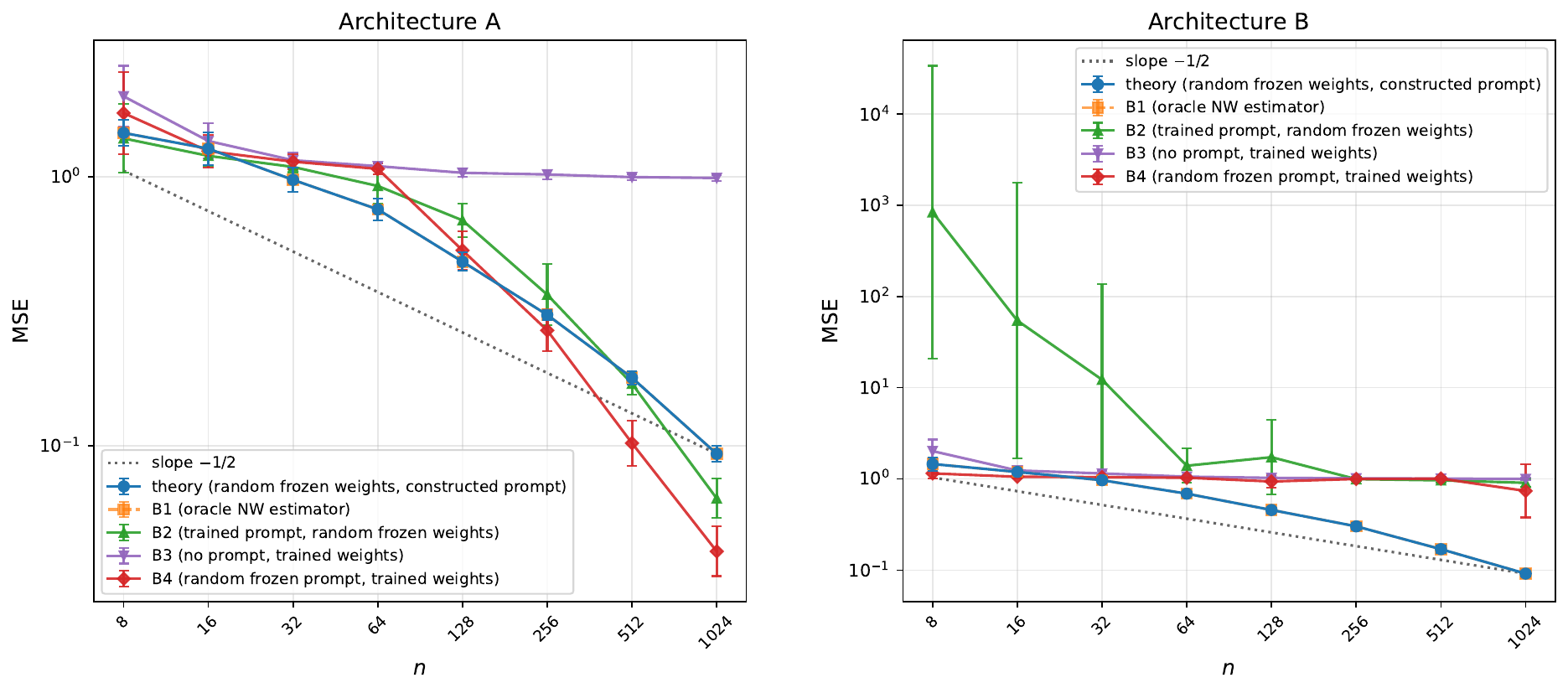}
\caption{Log-log plots of MSE versus $n$ for both architectures and various baselines, with fixed $\dh=256$ throughout. $h$ and $\beta$ are set according to the theory: $h=n^{-1/4}$ and $\beta=3\sqrt{n}/2+\log_+(C_G)$ (or $\beta=3\sqrt{n}/2+\log_+(C_{G,\mathrm{aff}})$ for Architecture B). MSE is computed over $1000$ uniformly sampled test queries per seed. Lines show geometric means and error bars are 95\% Student-$t$ CIs (computed in log-space).}
\label{fig:minimax}
\end{figure}

In nearly all cases, the error is similar for small $n$, since the number of samples is insufficient to accurately capture the function. For both architectures, our closed-form construction of the prompt is virtually indistinguishable from the empirical kernel estimator; comparing with the predicted minimax slope of $n^{-1/2}$ supports the upper bound for sample size/prompt length scaling. Benchmark 3 (weight training alone with no prompt) maintains MSE $\approx 1$ throughout. This is unsurprising, since a single softmax attention head on a single token collapses to an affine map, which has insufficient capacity no matter how it is tuned. 

In Architecture A, Benchmark 2 (frozen weights, trained prompt) shows slightly sharper scaling than our constructed prompts; the construction is sufficient, but likely not tight. Benchmark 4 (frozen prompt, trained weights) performs remarkably well, with seemingly better scaling than Benchmark 2, although not by a visibly significant amount. This highlights the utility of random context as a computational scratchpad for attention, an interesting idea that we leave for future work. 

In Architecture B, Benchmarks 2 and 4 do not learn anything meaningful, with Benchmark 2 even performing worse than the flat Benchmark 3 at smaller $n$. This is almost certainly a numerical failure, since the unit-variance initializations for the frozen components scale the attention logits to a size that is computationally infeasible for gradient-based methods. We still expect our construction not to be tight, and this just shows that there are settings where an explicit construction can outperform direct computational optimization.

\paragraph{Cost of adaptive logit shift.}
Finally, we test the predicted prompt norm tradeoffs in $\dh$, as specified in \cref{thm:univapproxrandomattn,thm:univapproxrandomtransformer}. In particular, we see the impact of the adaptive logit shift between Architecture A, which involves $C_G$, and Architecture B, which instead uses $C_{G,\mathrm{aff}}$. The results are presented in \cref{fig:tradeoff}.

\begin{figure}[h]
\centering
\includegraphics[width=0.9\linewidth]{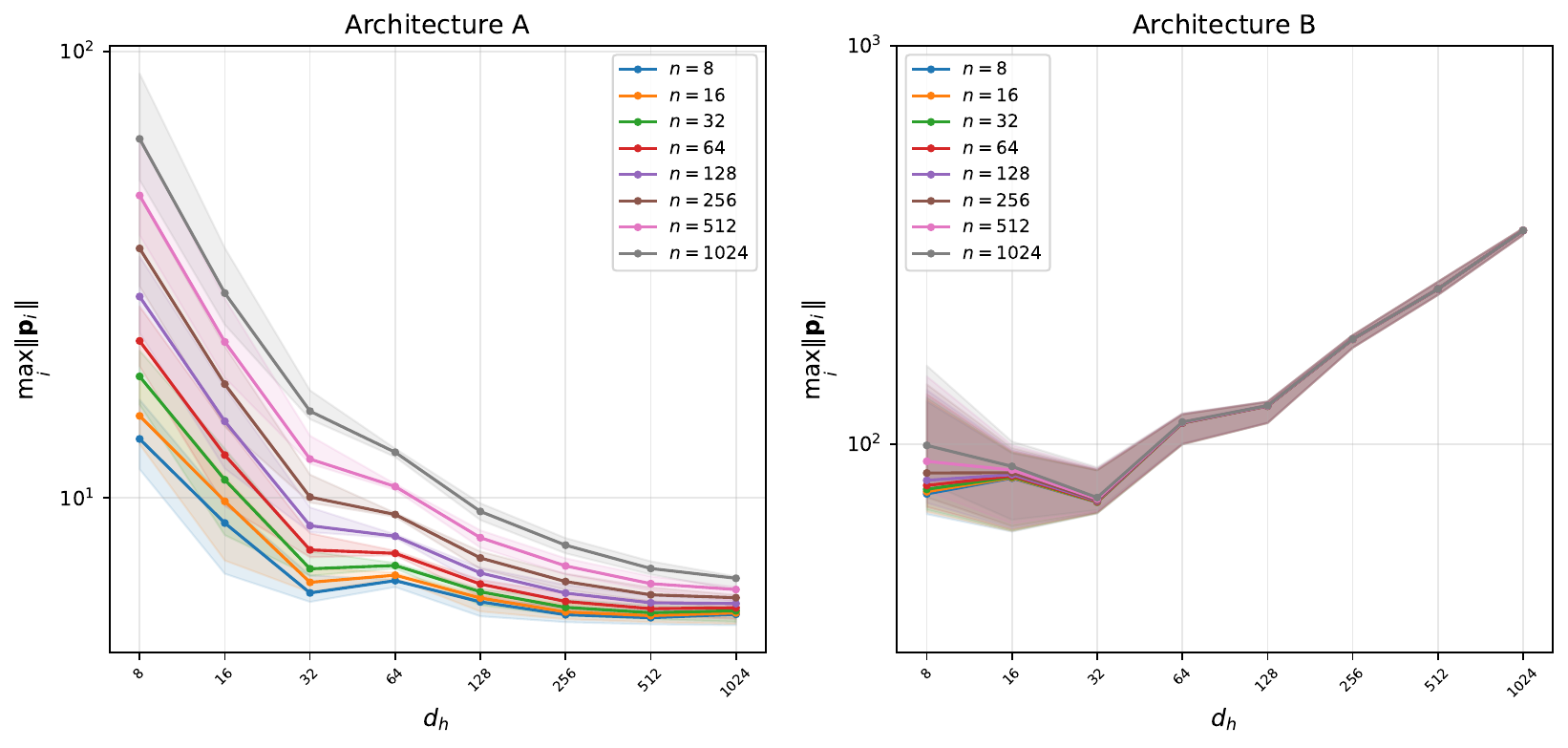}
\caption{Log-log plots of maximum prompt token norm of constructed soft prompt versus $\dh$ for both architectures across various $n$. We use the same explicitly computed choices for $h$ and $\beta$ as in \cref{fig:minimax}. Lines show medians and shaded regions show IQR.}
\label{fig:tradeoff}
\end{figure}

The prompt norm size in Architecture A is essentially monotone decreasing and flattening, aligning with the predicted $O(\dh^{-1/2}+1)$ rate. In the Architecture B result, the rate of $O(\dh^{-1}+\sqrt{\dh})$ manifests as a U-shaped curve, where the size initially decreases before rising linearly after a certain threshold.

\section{Conclusion}
We showed that a single-layer softmax attention network with frozen random weights is a universal approximator under soft prompting: for a H\"{o}lder function on a manifold, we give an explicit construction of a soft prompt, as the minimum-norm solution to a linear system, that steers the network to emulate the Nadaraya-Watson kernel estimator, inheriting minimax-optimal
rates that depend primarily on intrinsic dimension. The required rank condition on the weights holds almost surely under Gaussian initialization. We tracked the required soft prompt norm under three regimes: deterministic attention parameters, random attention parameters, and fully random transformers. 

Our results demonstrate the expressivity of prompting in an extreme case, leveraging the computational power of a single random attention head. The result is based on an explicit construction that depends on the realized transformer, directly connecting softmax attention and kernel methods through the prompt.

\paragraph{Future work}
Our work has focused on a heavily simplified transformer model, which we have proved is sufficient for our purposes. An interesting question is whether we can incorporate other elements of transformers, such as multi-head attention, deeper networks, feedforward components, residual connections, causal masking, and layer norm. In traditional weight-based approximation theory, most additional components only increase expressivity of networks, but it is not entirely obvious if that is also true in the prompting setting with random weights. In terms of the function class, extending to sequence-to-sequence functions is a natural next step. 

Our theoretical bounds were derived with the soft prompt construction and are not tight, as the numerical results suggest. The construction also utilizes one prompt token per support point; would fewer tokens be able to realize a comparable estimator, and is there a lower bound for a given accuracy?

\bibliographystyle{plainnat}
\bibliography{refs}

\end{document}